\documentclass[twoside]{article}

\usepackage[accepted]{aistats2020}

\usepackage{graphicx}
\usepackage{subcaption}
\usepackage{bm}
\usepackage{amsfonts}
\usepackage[colorinlistoftodos]{todonotes}
\definecolor{citrine}{rgb}{0.89, 0.82, 0.04}
\definecolor{blued}{RGB}{70,197,221}
\definecolor{applegreen}{rgb}{0.55, 0.71, 0.0}
\definecolor{flame}{rgb}{0.89, 0.35, 0.13}
\usepackage{amsthm}
\usepackage{amsmath}
\usepackage{amssymb}    
\usepackage{algorithm}
\usepackage{algorithmic}
\usepackage{mathtools}
\usepackage{booktabs}
\usepackage{array}
\usepackage{natbib}
\usepackage{amssymb}
\usepackage{bbold}
\usepackage{thmtools}
\usepackage{thm-restate}
\usepackage{pgfplots}
\usepackage{pgfplotstable}
\usepgfplotslibrary{groupplots}
\usetikzlibrary{matrix}
\usepackage{pifont}

\newtheorem{definition}{Definition}
\newtheorem{assumption}{Assumption}

\newtheorem{theorem}{Theorem}
\newtheorem{lemma}{Lemma}
\newtheorem{remark}{Remark}

\newcommand{\E}{\mathbb{E}}

\newcommand{\expec}[1]{\mathbb{E}\left[#1\right]}
\newcommand{\prob}[1]{\mathbb{P}\left\{#1\right\}}
\newcommand{\indi}[1]{\mathbb{1}\left\{#1\right\}}

\newcommand{\argmax}{\operatornamewithlimits{argmax}}
\newcommand{\arginf}{\operatornamewithlimits{arginf}}

\newcommand{\A}{\mathcal{A}}
\newcommand{\conf}{\tilde{\Theta}}
\newcommand{\Thetall}{\Theta^{all}}

\definecolor{green}{rgb}{0.0, 0.5, 0.0}

\begin{document}

%

%

\twocolumn[

\aistatstitle{A Novel Confidence-Based Algorithm for Structured Bandits}

\aistatsauthor{Andrea Tirinzoni \And Alessandro Lazaric \And Marcello Restelli}
\aistatsaddress{Politecnico di Milano \And Facebook AI Research  \And Politecnico di Milano}
]

\begin{abstract}

We study finite-armed stochastic bandits where the rewards of each arm might be correlated to those of other arms. We introduce a novel phased algorithm that exploits the given structure to build confidence sets over the parameters of the true bandit problem and rapidly discard all sub-optimal arms. In particular, unlike standard bandit algorithms with no structure, we show that the number of times a suboptimal arm is selected may actually be reduced thanks to the information collected by pulling other arms. Furthermore, we show that, in some structures, the regret of an anytime extension of our algorithm is uniformly bounded over time. For these constant-regret structures, we also derive a matching lower bound. Finally, we demonstrate numerically that our approach better exploits certain structures than existing methods.

\end{abstract}


\section{Introduction}

The widely studied multi-armed bandit (MAB) \citep{lai1985asymptotically,bubeck2012regret} problem is one of the simplest sequential decision-making settings in which a learner faces the exploration-exploitation dilemma. At each time $t$, the learner chooses an \emph{arm} $I_t$ from a finite set $\A$ and receives a random \emph{reward} $X_t$ whose unknown distribution depends on the chosen arm. The goal is to maximize the cumulative reward (or, equivalently, to minimize the regret w.r.t.\ the best arm) over a horizon $n$, which requires the agent to trade off between \emph{exploring} arms to understand their uncertain outcomes and \emph{exploiting} those that have performed best in the past.

The classic MAB problem, in which the rewards of the different arms are uncorrelated, is now theoretically well understood. In their seminal paper, \cite{lai1985asymptotically} provided the first asymptotic problem-dependent lower bound on the regret. Several simple yet near-optimal strategies have then been proposed, such as UCB1~\citep{auer2002finite}, Thompson Sampling \citep[TS,][]{thompson1933likelihood}, and KL-UCB~\citep{garivier2011kl}. However, the assumption that the arms are uncorrelated might be too general. In many applications, such as recommender systems or health-care, arms exhibit known structural properties that bandit algorithms could exploit to significantly speed-up the learning process.\footnote{In recommender systems, it is often possible to cluster users in a few \textit{types} based on their preferences. Once the \textit{type} of user is known, the value of each item is fixed.} 

Several specific structures have been addressed in the literature. Linear bandits are a well-known example, in which the mean reward of each arm is a linear function of some unknown parameter. Several algorithms have been proposed for these settings, such as extensions of UCB \citep{abbasi2011improved} and TS \citep{agrawal2013thompson,abeille2017linear}. However, these approaches, mostly based on the optimism in the face of uncertainty (OFU) principle, have been proved not asymptotically optimal \citep{lattimore2017end}. Examples of other specific structures include combinatorial bandits \citep{cesa2012combinatorial}, Lipschitz bandits \citep{magureanu2014lipschitz}, ranking bandits \citep{combes2015learning}, unimodal bandits \citep{yu2011unimodal}, etc.

Recently, there has been a growing interest in designing bandit strategies to exploit general structures, where the learner is provided with a subset of all possible bandit problems containing the (unknown) problem she has to face. The structured UCB algorithm, proposed almost-simultaneously by \cite{lattimore2014bounded} and \cite{azar2013sequential}, applies the OFU principle to general structures. \cite{atan2018global} proposed a greedy algorithm for the special case where all arms are informative, while \cite{wang2018regional} extended these settings to consider correlations only within certain groups of arms and independence among them. \cite{gupta2018exploiting} generalized UCB and TS to exploit the structure and quickly identify sub-optimal arms. One of the interesting findings of these works is that, in some structures, constant regret (i.e., independent of $n$) is possible. 
In the remainder, we shall call these strategies \emph{confidence-based} since they explicitly maintain the uncertainties about the true bandit and use these to trade-off exploration/exploitation. Although conceptually simple, confidence-based strategies are typically hard to design and analyze in a fully structure-aware manner. In fact, in structured problems, pulling an arm provides not only a sample of its mean, but also information about the bandit problem itself through the knowledge of the overall structure. In turn, information about the problem itself potentially allow to refine the estimates of the means of \emph{all} arms. \cite{combes2017minimal} made a significant step in exploiting this interplay between arms and bandit problems in the very definition of the algorithm itself. The authors derived a structure-aware lower bound characterizing the optimal pull counts as the solution to an optimization problem. Their algorithm, OSSB, approximates this solution and achieves asymptotic optimality for any general structure. However, since the lower bound depends on the true (unknown) bandit at hand, this approach requires to force some exploration to guarantee a sufficiently accurate solution. For this reason, we shall call this kind of strategy \emph{forced-exploration}. Compared to confidence-based ones, it can be intractable in many structures and it remains an open question how well it performs in finite time.

In this paper, we focus on the widely-applied confidence-based strategies for structured bandits. Our contributions are as follows. \textbf{1)} We propose an algorithm running through phases. At the beginning of each phase, the set of bandit models compatible with the confidence intervals computed so far is built and the corresponding optimal arms are repeatedly pulled in a round-robin fashion, until the end of the phase. For this strategy, we prove an upper bound on the expected regret that, compared to existing bounds, better shows the potential benefits of exploiting the structure. The key finding is that the number of pulls to a sub-optimal arm $i$ can be significantly reduced by exploiting the information obtained while pulling other arms, and notably the arm that is most informative for this purpose, i.e., the arm for which the mean of the true bandit differs the most from that of any other bandit in which arm $i$ is optimal. This is in contrast to existing methods, which rely exclusively on the samples obtained from arm $i$ to identify its suboptimality (a property that is true for the unstructured settings). \textbf{2)} Since our algorithm requires to know the horizon $n$, we design a practical anytime extension for which, under the same assumptions as in \citep{lattimore2014bounded}, we derive a constant-regret bound with a better scaling in the relevant structure-dependent quantities. \textbf{3)} For certain structures that satisfy the aforementioned assumption, we also derive a matching lower bound that shows the optimality of our algorithm in the constant-regret regime. \textbf{4)} We report numerical simulations in some simple illustrative structures that confirm our theoretical findings.


\section{Preliminaries}\label{sec:preliminaries}

We follow similar notation and notions to formalize MAB with structure as in~\citep{agrawal1988asymptotically,graves1997asymptotically,burnetas1996optimal,azar2013sequential,lattimore2014bounded,combes2017minimal}. 
We denote by $\Thetall$ the collection of all bandit problems $\theta$ with a set of arms $\A$ and whose reward distributions $\{\nu_i\}_{i\in\A}$ are bounded in $[0,1]$\footnote{As usual, this assumption can be relaxed to sub-Gaussian noise with no additional complications.}. We refer to each $\theta\in\Thetall$ as a \textit{bandit (problem)}, or \textit{model}. We denote by $\mu_i(\theta)$ the mean reward of arm $i$ in model $\theta$ and let $\mu^*(\theta) := \max_{i\in\A} \mu_i(\theta)$. For the sake of readability, we assume that the corresponding optimal arm, $i^*(\theta) := \argmax_{i\in\A} \mu_i(\theta)$, is unique for all models. The sub-optimality gap of arm $i\in\A$ is $\Delta_i(\theta) := \mu^*(\theta) - \mu_i(\theta)$, while the model gap w.r.t.\ $\theta'\in\Thetall$ is $\Gamma_i(\theta, \theta') := |\mu_i(\theta) - \mu_i(\theta')|$. It is known that the gaps $\Delta$ characterize the complexity of a bandit problem in the unstructured case. As we shall see, the model gaps $\Gamma$ play the analogous role in structured problems. A \emph{structure} $\Theta \subseteq \Thetall$ is a subset of possible models. For instance, a linear structure is a set of models whose mean rewards can be written as a linear combination of given features. We denote by $\A^*(\Theta)$, abbreviated $\A^*$ when $\Theta$ is clear from context, the set of arms that are optimal for at least one model in $\Theta$, while $\Theta_i^*$ is the set of models in which arm $i$ is optimal.

Let $\theta^*\in\Thetall$ be the \textit{true} model and $\Omega := \{\Theta' \subseteq \Thetall \ |\ \theta^*\in\Theta'\}$. A (structured) bandit algorithm $\pi$ receives as input a structure $\Theta \in\Omega^{}$ and defines a strategy for choosing the arm $I_t$ given the history $H_{t-1} = (I_1,X_1,\dots,I_{t-1},X_{t-1})$\footnote{Whenever $\pi$ receives as input $\Theta^{all}$, it reduces to the standard MAB case.}. Our performance measure is the expected regret after $n$ steps,
\begin{equation*}
R_n^\pi(\theta^*, \Theta) := n \mu^*(\theta^*) - \E_{\pi,\theta^*} \left[ \sum_{t=1}^n \mu_{I_t}(\theta^*) \right].
\end{equation*}
Note that the regret depends on $\Theta$ through the strategy $\pi$. In the remaining, whenever $\theta$ is dropped from a model-dependent quantity, we implicitly refer to $\theta^*$.


\paragraph*{Structured UCB}

Structured UCB (SUCB)\footnote{The algorithm was originally called UCB-S by \cite{lattimore2014bounded} and mUCB by \cite{azar2013sequential}.} is a natural extension of the OFU principle to general structures and it reduces to UCB whenever the structure $\Theta$ provided as input is the set of all possible bandit problems (i.e., $\Theta^{all}$). At each step $t$, the algorithm builds a confidence set $\conf_t \subseteq \Theta$ containing all the models compatible with the confidence intervals built for each arm and it pulls the optimistic arm $I_t = \argmax_{i\in\A}\sup_{\theta\in\conf_t}\mu_i(\theta)$. While taking the optimistic arm ensures that ``good'' arms are selected, refining the confidence set $\conf_t$ allows to exploit the structure to possibly discard arms more rapidly. \cite{lattimore2014bounded} derived the same upper bound to the regret as the one of UCB without making any assumption on set $\Theta$. On the other hand, \cite{azar2013sequential} derived a more structure-aware bound, but only for finite $\Theta$. The next theorem combines the best of these analyses (see proof in App.~\ref{app:sucb-proof}). We first introduce two quantities that conveniently characterize the number of samples needed to distinguish between models. For any $\Theta' \in \Omega^{}$ and $\A' \subseteq \A$, we define:
\begin{equation}\label{eq:Psi}
\Psi(\Theta',\A') := \inf_{\theta\in\Theta'}\max_{j\in\A'}\Gamma_j^2(\theta,\theta^*),
\end{equation}
\begin{equation}\label{eq:psi}
\psi(\Theta',\A') := \arginf_{\theta\in\Theta'}\max_{j\in\A'}\Gamma_j^2(\theta,\theta^*).
\end{equation}
It is known that the number of pulls to an arm $i$ that are sufficient to distinguish between $\theta^*$ and any $\theta$ is bounded as $\mathcal{O}(1/\Gamma_i^2(\theta,\theta^*))$ with high-probability \citep{azar2013sequential}. 
Then, we can interpret $\Psi(\Theta',\A')$ as proportional to the inverse number of pulls required from the \emph{most effective} arm in $\A'$ to distinguish $\theta^*$ from the model $\psi(\Theta',\A')$, i.e., the bandit problem in $\Theta'$ that is most similar to $\theta^*$ in terms of model gaps. For this reason, we refer to $\psi(\Theta',\A')$ as the \emph{hardest model} in $\Theta'$ using arms in $\A'$. Finally, we define the following sets of optimistic models w.r.t. $\theta^*$: $\Theta^{+} := \{ \theta \in \Theta : \mu^*(\theta) > \mu^*(\theta^*) \}$ and $\Theta_{i}^+ := \{ \theta \in \Theta^{+} : i^*(\theta) = i \}$.
	
\begin{restatable}[]{theorem}{thsucb}\label{th:sucb}
	There exist constants $c,c'>0$ such that for any model $\theta^*\in\Theta^{all}$ and any structure $\Theta \in \Omega^{}$, the expected regret at time $n$ of the SUCB algorithm~\citep{lattimore2014bounded} is upper-bounded as
	\begin{equation*}
	R_n^{\text{SUCB}}(\theta^*, \Theta) \leq \sum_{i \in \mathcal{A}^* \setminus \{i^*\}} \frac{c\Delta_i(\theta^*)\log n}{\Psi(\Theta_{i}^+, \{i\})} + c'.
	\end{equation*}
\end{restatable}

This result shows that SUCB is able to leverage the knowledge of $\Theta$ to improve over UCB, which relies only on $\Theta^{all}$. First, the summation is limited to arms that are optimal in at least one model in $\Theta$. Second, the number of pulls of a sub-optimal arm $i$ depends on the model gap $\Gamma_i(\theta_i^+, \theta^*)$ w.r.t. the \emph{hardest model} $\theta_i^+ = \psi(\Theta_i^+, \{i\})$. This measures the number of pulls necessary to distinguish $\theta_i^+$ from $\theta^*$ by pulling $i$. This gap can be much larger than the sub-optimality gap $\Delta_i(\theta^*)$ which appears in unstructured settings (e.g., UCB), thus significantly reducing the final regret.

While UCB-based algorithms are proved to be optimal (i.e., they match the asymptotic lower bound of~\citet{lai1985asymptotically}), evaluating the optimality of Thm.~\ref{th:sucb} is less obvious. We need to first introduce a specific type of structures. We say that $\Theta$ is a worst-case structure if it belongs to the set
\begin{align*}
\Omega^{\text{wc}} := \left\{ \Theta \in \Omega^{}\ |\ \forall i\neq i^* : \Psi(\Theta_i^+,\{i\}) = \Psi(\bar{\Theta}_i^+,\{i\}) \right\},
\end{align*}
where $\bar{\Theta}_i^+ := \{\theta\in\Theta_i^+ | \max_{j\neq i}\Gamma_j(\theta,\theta^*)=0\}$ is the subset of optimistic models that are indistinguishable from $\theta^*$ except in their optimal arm. Thus, a worst-case structure is such that the hardest optimistic models cannot be distinguished from $\theta^*$ except in their optimal arm. Note that $\Thetall \in \Omega^{\text{wc}}$. An asymptotic lower bound for these structures has already been provided by \cite{burnetas1996optimal}. We state here the version for Gaussian bandits with fixed variance equal to 1 to facilitate comparison with the upper-bounds.
\begin{theorem}[\cite{burnetas1996optimal}]\label{th:bk}
For any $\Theta\subseteq\Omega^{\text{wc}}$ and uniformly convergent strategy $\pi$,
\begin{align*}
\liminf_{n \rightarrow \infty} \frac{R_n^\pi(\theta^*,\Theta)}{\log n} \geq \sum_{i\in\A^*\setminus\{i^*\}} \frac{\Delta_i(\theta^*)}{\Psi(\Theta_i^+, \{i\})}.
\end{align*}
\end{theorem}
We refer the reader to \citep{garivier2018explore} for a simple proof and the definition of uniformly convergent strategies. The immediate consequence of Theorem \ref{th:bk} is that SUCB is asymptotically order-optimal for all worst-case structures.

\section{Structured Arm Elimination}\label{sec:sae}

Our structured arm elimination (SAE) strategy (Algorithm \ref{alg:isucb}) is a phased algorithm inspired by Improved UCB \citep{auer2010ucb}. In each phase $h$, the algorithm keeps a confidence set containing the models such that the mean of each arm $i$ does not deviate too much from the empirical one $\hat{\mu}_{i,h-1}$ according to its number of pulls $T_i(h-1)$, both computed at the end of the previous phase. Then, all active arms (i.e., those that are optimal for at least one of the models in the confidence set) are played until a well-chosen pull count is reached. Such count is computed to ensure that all models that are sufficiently distant from the target $\theta^*$ (according to an exponentially-decaying removal threshold $\tilde{\Gamma}_h$) are discarded from the confidence set. Once all the models in which a certain arm $i\in\A$ is optimal have been eliminated, $i$ is labeled as inactive and no longer pulled. Algorithm \ref{alg:isucb} can be applied to any set of models (not only finite ones) as far as we can determine the set of optimal arms at each step. This is an optimization problem that can be solved efficiently for, e.g., linear, piecewise-linear, and convex structures, while it becomes intractable in general.

\begin{algorithm}[t]
\small
\caption{Structured Arm Elimnation (SAE)} \label{alg:isucb}
\begin{algorithmic}[1]
\REQUIRE Set of models $\Theta$, horizon $n$, scalars $\alpha > 0, \beta \geq 1$
\vspace{0.001cm}
\STATE{\textbf{Initialization}:}
\STATE{$\tilde{\Theta}_0 \leftarrow \Theta$ (confidence set)}
\STATE{$\tilde{\mathcal{A}}_0 \leftarrow \mathcal{A}^*(\Theta)$ (set of active arms)}
\STATE{$\tilde{\Gamma}_0 \leftarrow 1$ (removal threshold)}
\STATE{\textbf{Foreach phase} $h=0,1,\dots$ \textbf{do}}
\STATE{Play all active arms in a round-robin fashion until $\left\lceil\frac{\alpha\log n}{\tilde{\Gamma}_h^2}\left(1 + \frac{1}{\beta}\right)^2\right\rceil$ pulls are reached for all $i\in\tilde{A}_h$}
\STATE{Update confidence set:\\ $\tilde{\Theta}_{h+1} \leftarrow \left\{ \theta\in\Theta \ \big|\ \forall i \in \mathcal{A} : |\hat{\mu}_{i,h} - \mu_i(\theta)| < \sqrt{\frac{\alpha\log n}{T_i(h)}}\right\}$\footnotemark}
\STATE{Update set of active arms: $\tilde{A}_{h+1} = \mathcal{A}^*(\tilde{\Theta}_{h+1}) \cap \tilde{A}_h$}
\STATE{Decrease removal threshold: $\tilde{\Gamma}_{h+1} \leftarrow \frac{\tilde{\Gamma}_{h}}{2}$}

\STATE{\textbf{End}}
\end{algorithmic}
\end{algorithm}

Note that SAE is not an optimistic algorithm since it might pull arms that are never optimistic w.r.t. $\theta^*$. This property is due to the phased nature of the algorithm, such that no \textit{optimistic bias} in selecting the active arms is used, unlike in SUCB. While in unstructured problems SUCB and SAE reduce to UCB and improved UCB, respectively, and have similar regret guarantees (i.e., each arm is pulled roughly the same amount of times in the two algorithms), in structured problems they may behave very differently, as we shall see in the next examples.
\footnotetext{We implicitly assume this condition to hold for arms that have never been pulled before.}

\subsection{Examples}

Figure \ref{fig:example} presents two simple structures in which SUCB and SAE significantly differ. The model set is divided in different regions. Since all bandits in the same region have, for the purpose of our discussion, the same properties, we call $\theta_1$ any model in the first part, $\theta_2$ any model in the second, and so on. Note that the following comments hold for an ideal realization in which certain high-probability events occur.

In the structure of Figure \ref{fig:example}\emph{(left)}, arm $2$ is never optimistic since its mean is always below the value of the optimal arm $\mu_1(\theta_1)$. Therefore, SUCB never pulls it and needs only to discard the optimistic arm $3$. This, in turn, takes $\mathcal{O}(1/\Gamma_3^2(\theta_1,\theta_2))$ pulls of such arm, which can be rather large. Since SAE pulls also arm $2$, the large gap $\Gamma_2(\theta_1, \theta_2)$ ($\Gamma_2$ in the figure) allows to discard arm $3$ much sooner. From the definition of the algorithm, SAE also needs to discard arm $2$. Once again, this can be done quickly due to the large gap $\Gamma_1(\theta_1,\theta_3)$ and the fact that the optimal arm $1$ is always pulled.

In the structure of Figure \ref{fig:example}\emph{(right)}, the optimistic bias makes SUCB pull the arms starting from the one with the highest value, arm $2$, downwards to the optimal one, arm $1$. Since the gap $\Gamma_2(\theta_1, \theta_3)$ ($\Gamma_2$ in the figure) is larger than $\Gamma_2(\theta_1, \theta_4)$, SUCB implicitly discards $\theta_3$, and so arm $4$, before arm $2$. Thus, once both these arms have been eliminated, the algorithm takes $\mathcal{O}(1/\Gamma_3^2(\theta_1,\theta_2))$ pulls of arm $3$ to discard the arm itself. By simultaneously pulling all four arms, SAE discards arm $3$ first using the pulls of arm $4$ (the one prematurely discarded by SUCB) due to the large gap $\Gamma_4(\theta_1,\theta_2)$ ($\Gamma_4$ in the figure). Finally, the deletion of the remaining two sub-optimal arms occurs with the same number of pulls as SUCB, and it can be verified that the overall regret is much smaller.

\begin{figure} 
\begin{tikzpicture}
\begin{groupplot}[group style={group size=2 by 1,
horizontal sep=0.5cm,
group name=myplot},
width=5cm,
height=5cm,
xtick style={draw=none},
ytick style={draw=none}
]

\nextgroupplot[
xmin=0, xmax=1, ymin=0, ymax=1, xtick={0, 1}, ytick={0,1}, xlabel=$\theta$, ylabel=$\mu$,
axis lines = center, 
axis line style = very thick,
axis on top=true,
extra y ticks={0},
x label style={at={(axis description cs:0.5,-0.02)},anchor=north}, 
y label style={at={(axis description cs:-0.15,.5)},anchor=south}
]

  \filldraw[black!3] (axis cs:0.01,0.01) rectangle (axis cs:0.33,1);
  \addplot[black!10, solid] coordinates { (0.33,0) (0.33,1) };
  \addplot[black!10, solid] coordinates { (0.66,0) (0.66,1) };
  
  \addplot[black!40, solid] coordinates { (0.16,0.22) (0.16,0.78) };
  \addplot[black!40, solid] coordinates { (0.14,0.22) (0.18,0.22) };
  \addplot[black!40, solid] coordinates { (0.14,0.78) (0.18,0.78) };
  \node[black!60] at (axis cs:0.25,0.5) {\small $\Gamma_2$};
  
  \addplot[green, dashed, line width=1pt] coordinates { (0,0.8) (0.33,0.8) };
  \addplot[green, dashed, line width=1pt] coordinates { (0.33,0.2) (0.66,0.2) };
  \addplot[green, dashed, line width=1pt] coordinates { (0.66,0.8) (1,0.8) };
  
  \addplot[blue, solid, line width=1pt] coordinates { (0,0.85) (0.33,0.8) };
  \addplot[blue, solid, line width=1pt] coordinates { (0.33,0.8) (0.66,0.4) };
  \addplot[blue, solid, line width=1pt] coordinates { (0.66,0.4) (1,0.4) };
  
  \addplot[red, dotted, line width=1pt] coordinates { (0,0.6) (0.33,0.8) };
  \addplot[red, dotted, line width=1pt] coordinates { (0.33,0.86) (0.66,0.86) };
  \addplot[red, dotted, line width=1pt] coordinates { (0.66,0.8) (1,0.6) };

\nextgroupplot[
xmin=0, xmax=1, ymin=0, ymax=1, xtick={0, 1}, ytick=\empty, xlabel=$\theta$, ylabel=\empty,
axis lines = center, 
axis line style = very thick,
axis on top=true,
x label style={at={(axis description cs:0.5,-0.02)},anchor=north}, 
y label style={at={(axis description cs:-0.03,.5)},anchor=south}
]

  \filldraw[black!3] (axis cs:0.01,0.01) rectangle (axis cs:0.25,1);
  \addplot[black!10, solid] coordinates { (0.25,0) (0.25,1) };
  \addplot[black!10, solid] coordinates { (0.5,0) (0.5,1) };
  \addplot[black!10, solid] coordinates { (0.75,0) (0.75,1) };
  
  \addplot[black!40, solid] coordinates { (0.12,0.12) (0.12,0.48) };
  \addplot[black!40, solid] coordinates { (0.10,0.12) (0.14,0.12) };
  \addplot[black!40, solid] coordinates { (0.10,0.48) (0.14,0.48) };
  \node[black!60] at (axis cs:0.2,0.3) {\small $\Gamma_4$};
  
  \addplot[black!40, solid] coordinates { (0.62,0.42) (0.62,0.68) };
  \addplot[black!40, solid] coordinates { (0.60,0.42) (0.64,0.42) };
  \addplot[black!40, solid] coordinates { (0.60,0.68) (0.64,0.68) };
  \node[black!60] at (axis cs:0.56,0.5) {\small $\Gamma_2$};
  
  \addplot[blue, solid, line width=1pt] coordinates { (0,0.8) (0.25,0.8) };
  \label{plot:arm_one}
  \addplot[blue, solid, line width=1pt] coordinates { (0.25,0.8) (0.5,0.8) };
  \addplot[blue, solid, line width=1pt] coordinates { (0.5,0.8) (0.75,0.8) };
  \addplot[blue, solid, line width=1pt] coordinates { (0.75,0.8) (1,0.8) };
  
  \addplot[green, dashed, line width=1pt] coordinates { (0,0.7) (0.25,0.7) };
  \label{plot:arm_two}
  \addplot[green, dashed, line width=1pt] coordinates { (0.25,0.7) (0.5,0.7) };
  \addplot[green, dashed, line width=1pt] coordinates { (0.5,0.4) (0.75,0.4) };
  \addplot[green, dashed, line width=1pt] coordinates { (0.75,0.92) (1,0.92) };
  
  \addplot[red, dotted, line width=1pt] coordinates { (0,0.6) (0.25,0.6) };
  \label{plot:arm_three}
  \addplot[red, dotted, line width=1pt] coordinates { (0.25,0.84) (0.5,0.84) };
  \addplot[red, dotted, line width=1pt] coordinates { (0.5,0.6) (0.75,0.6) };
  \addplot[red, dotted, line width=1pt] coordinates { (0.75,0.6) (1,0.6) };
  
  \addplot[violet, dashdotted, line width=1pt] coordinates { (0,0.5) (0.25,0.5) };
  \label{plot:arm_four}
  \addplot[violet, dashdotted, line width=1pt] coordinates { (0.25,0.1) (0.5,0.1) };
  \addplot[violet, dashdotted, line width=1pt] coordinates { (0.5,0.88) (0.75,0.88) };
  \addplot[violet, dashdotted, line width=1pt] coordinates { (0.75,0.5) (1,0.5) };
  
\end{groupplot}

\path (myplot c2r1.south west|-current bounding box.south)--
coordinate(legendpos)
(myplot c2r1.south west|-current bounding box.south);
\matrix[matrix of nodes,
anchor=north,
inner sep=0.2em,
font=\scriptsize\rmfamily
]at([yshift=0.5ex, xshift=-4ex]legendpos)
{ 
\ref{plot:arm_one}& Arm $1$ &
\ref{plot:arm_two}& Arm $2$ &
\ref{plot:arm_three}& Arm $3$ &&
\ref{plot:arm_four}& Arm $4$ \\
};

\end{tikzpicture}
\caption{Two structures in which SUCB and SAE significantly differ. The true model is any in the shaded region. \emph{(left)} SUCB never pulls an informative arm. \emph{(right)} SUCB discards an informative arm too early.}
\label{fig:example}
\end{figure}
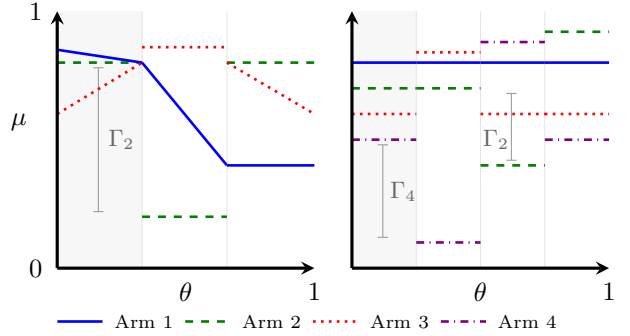

\subsection{Regret Analysis}

In order to upper bound the regret of Alg.~\ref{alg:isucb}, we need to characterize the arms pulled in each phase, which are specified by the sets of active arms $\big\{\tilde{\A}_h\big\}_h$. Since these sets are random quantities, we cannot study them directly. Instead, we introduce a deterministic sequence of active arm sets $\{\A_h\}_h$ that effectively works as a proxy for $\big\{\tilde{\A}_h\big\}_h$ and, under certain high-probability events, allows us to define how many samples are needed for arms to be discarded. We now provide intuitions (made formal in the proof of the regret bound) on how such sequence is built. Clearly, we have $\A_0 = \tilde{A}_0 = \A^*(\Theta)$ by definition. Since all arms in $\A_0$ are pulled in $h=0$, and recalling the meaning of $\Psi$ (Equation \ref{eq:Psi}), our well-chosen pull counts are \emph{sufficient} to prove that all arms $i$ such that $\Psi(\Theta_i^*,\A_0) \geq \tilde{\Gamma}_0^2$ are discarded. Let us call the set of these discarded arms $\bar{\A}_0$ and apply this reasoning inductively by setting $\A_1 = \A_0 \setminus \bar{\A}_0$. Unfortunately, it is general not possible to conclude that $\A_1 = \tilde{\A}_1$ since other arms might be discarded. Therefore, we build an additional set $\underline{\A}_h$ of those arms that are guaranteed to be active in phase $h$. The main intuition is that, if we can prove that certain arms are still active, we can also show that the algorithm uses their \emph{information} (i.e., the model-gaps) to discard certain other arms/models faster. Imagine that an oracle provides us with the set $\underline{\A}_h$. Then, for $h \geq 0$ we have
\begin{align*}
\bar{\mathcal{A}}_h := \left\{ i \in \mathcal{A}_h \ |\ \tilde{\Gamma}_h \leq \inf_{\theta\in\Theta^*_i} \max_{j \in \underline{\mathcal{A}}_h \cup \{i\}} \Gamma_j(\theta, \theta^*) \right\},
\end{align*}
with $\mathcal{A}_0 = \mathcal{A}^*(\Theta)$ and $\mathcal A_{h+1} = \mathcal A_{h} \setminus\bar{\A}_{h}$ for $h \geq 1$. Given these sets, we have $\underline{\mathcal{A}}_0 := \mathcal{A}^*(\Theta)$ and
\begin{align*}
\underline{\mathcal{A}}_h := \left\{ i \in \mathcal{A}_h \ |\ \tilde{\Gamma}_{h-1} > k_{\beta}\inf_{\theta\in\Theta^*_i} \max_{j \in \mathcal{A}^*(\Theta)} \frac{\Gamma_j(\theta, \theta^*)}{2^{[h-\bar{h}_j-1]_+}} \right\}
\end{align*}
for all $h \geq 1$, where $k_{\beta} := \frac{1}{\beta - 1}\sqrt{(\beta + 1)^2 + \frac{1}{\log n}}$ and $\bar{h}_j := \max_{h \in \mathbb{N}^+} \{h\ |\ j \in \A_h\}$ is the last phase in which arm $j$ is active in our deterministic sequence $\{\A_h\}_h$. This is essentially the set of arms for which the number of pulls to the active arms at the previous phase is below the removal threshold by a \emph{margin} (defined by $k_\beta$). Finally, we define the set of arms that are active in the last phase when $i$ is active as $\A_i^* = \underline{\mathcal{A}}_{\bar{h}_i} \cup \{i\}$.

The following theorem is the key result of this paper. It shows that the regret incurred by SAE for arm $i$ is inversely proportional to the maximum model-gap (taken over the set of arms that are active when arm $i$ is discarded) w.r.t. the hardest model in $\Theta_i^*$.

\begin{theorem}\label{th:regret-main}
Let $\beta \geq 1$, $\alpha = \beta^2$, $n \geq 64$, and $c_\beta := 4(1+\beta^2)$.
Then,
\begin{equation*}
R_n^{SAE}(\theta^*, \Theta) \leq \sum_{i \in \mathcal{A}^*\setminus\{i^*\}} \frac{c_\beta\Delta_i(\theta^*)\log n}{\Psi(\Theta_i^*, \A_i^*)} + 2|\mathcal{A}^*(\Theta)|.
\end{equation*}
\end{theorem}

One of the key novelties, and complications, in the proof (reported in App. \ref{app:proofs3}) is that, in order to carry out a fully structure-aware analysis, we do not only care about proving that sub-optimal arms are not pulled after certain phases, but also about guaranteeing that some arms are not discarded too early since their pulls might allow to discard other models/arms. The parameter $\beta$ plays an important role for this purpose.
In particular, $k_\beta$ controls the sets of arms that, with high probability, are guaranteed to be active at certain phases. For example, for large $n$, setting $\beta =3$ yields $k_{\beta} \simeq 2$, which in turn implies that $\underline{\mathcal A}_h$ is the set of arms such that $\tilde{\Gamma}_{h} > \inf_{\theta\in\Theta^*_i} \max_{j \in \mathcal{A}^*(\Theta)} \frac{\Gamma_j(\theta, \theta^*)}{2^{[h-\bar{h}_j-1]_+}}$. This is close to saying that all the arms that are not eliminated in phase $h$ are also active in such phase.

\subsection{Discussion}\label{sec:discussion}

First, as a sanity check, we verify that the regret bound of Theorem \ref{th:regret-main} is never worse than the one of UCB. That is, SAE is never negatively affected by the knowledge of the structure and, whenever applied to unstructured problems, the algorithm is, apart from multiplicative/additive constants, finite-time optimal.

\begin{restatable}[]{proposition}{propcomparisonucb}\label{prop:comparison-UCB}
The SAE algorithm is always sub-UCB, in the sense that there exist constants $c,c'>0$ such that its regret satisfies
\begin{align*}
R_n^{SAE}(\theta^*,\Theta) \leq \sum_{i\in\mathcal{A}\setminus \{i^*\}} \frac{c\log n}{\Delta_i(\theta^*)} + c'.
\end{align*}
\end{restatable}

The key property of Thm.~\ref{th:regret-main} is that the regret suffered for discarding a sub-optimal arm $i$ does not necessarily scale with the model gaps of such arm (i.e., $\Psi(\Theta_i^*, \{i\})$) but with those of the most effective arm in $\A_i^*$. Thus, compared to SUCB, in which the elimination of a model $\theta\in\Theta^*_i$ requires $\mathcal{O}(1/\Gamma_i^2(\theta,\theta^*))$ pulls of arm $i$, SAE needs only $\mathcal{O}(1/\max_{j\in\A_i^*}\Gamma_j^2(\theta,\theta^*))$, which is by definition always smaller. Note that, to be precise, SUCB can potentially eliminate models using the pulls of any arm since the confidence sets are built as in SAE. However, in general, it is not possible to prove the same regret bound since the optimism induces a specific pull order that might prevent the algorithm from choosing the arm with the largest model gap. Obviously, SAE does not know this arm in advance and, therefore, ensures it is pulled by choosing all active arms. However, the additional regret incurred to achieve this property can make the algorithm, in some cases, worse than SUCB. In fact, a key difference is that SUCB stops playing a sub-optimal arm $i$ when all optimistic models in $\Theta^+_i$ are discarded, while SAE needs to eliminate all models in which arm $i$ is optimal (even non-optimistic ones). Therefore, although SAE improves the elimination of all optimistic models, it suffers further regret for discarding non-optimistic ones and, in general, the two algorithms are not comparable. A special case are those structures in which the hardest models for each arm $i$ are in the optimistic set, $\psi(\Theta_i^*, \A_i^*) \in \Theta_i^+$, in which SAE improves over SUCB. These \emph{optimistic} structures are defined as:
\begin{align*}
\Omega^{\text{opt}} := \{ \Theta \in \Omega \ |\ &\forall i \neq i^* : \Psi(\Theta_i^+, \A_i^*) = \Psi(\Theta_i^*, \A_i^*) \}.
\end{align*}
\begin{restatable}[]{proposition}{propcomparisonsucb}\label{prop:comparison-SUCB}
If $\Theta \in \Omega^{\text{opt}}$, SAE is sub-SUCB, in the sense that its regret can be upper bounded by the one of Theorem \ref{th:sucb}.
\end{restatable}
Since SUCB is order-optimal in $\Omega^{\text{wc}}$ and SAE is sub-SUCB in $\Omega^{\text{opt}}$, Theorem \ref{th:bk} immediately implies that SAE is order optimal in $\Omega^{\text{wc}} \cap \Omega^{\text{opt}}$. Although we are able to guarantee the optimality in less cases, Proposition \ref{prop:comparison-SUCB} ensures that SAE improves over SUCB in a wide variety of structures. Unfortunately, we were not able to prove the optimality of our algorithm in any structure besides the worst-case ones.

\section{Anytime SAE and Constant Regret} \label{sec:anytime}

Algorithm \ref{alg:isucb} cannot be applied whenever the horizon $n$ is unknown, as the length of each phase explicitly depends on it. This has the additional drawback of preventing constant regret from being achieved since a $\log n$ term naturally appears in the resulting bound. As shown by \cite{lattimore2014bounded}, there exist structures in which constant regret can be obtained and it would be desirable for our strategy to exploit this fact. We, therefore, propose an anytime extension (Algorithm \ref{alg:isucb-anytime}). The idea is once again similar to the one by \cite{auer2010ucb}: we split the horizon into different \emph{periods} with exponentially increasing length. Therefore, in Algorithm \ref{alg:isucb-anytime}, and throughout this section, we overload our  notation by adding a superscript $k$ to denote the period of each period-dependent quantity. The key property is that our approach does not reset in each period (as \cite{auer2010ucb} do) but retains the last confidence sets. Though this makes the proofs more involved, we shall see that it allows us to guarantee a constant regret. One can see the analogy between our non-resetting phased approach and the standard way of handling unknown horizons in online algorithms. In the latter case, we typically replace $\log n$ with $\log t$ in the confidence sets, while here we do the same with $\log \tilde{n}_k$. Then, after proving that certain high-probability events occur at each time/period, we can carry out the proofs without forcing any reset.
\begin{algorithm}[t]
\small
\caption{Anytime SAE (ASAE)} \label{alg:isucb-anytime}
\begin{algorithmic}[1]
\REQUIRE Set of models $\Theta$, scalars $\alpha > 0, \beta \geq 1, \eta > 0$
\vspace{0.1cm}
\STATE{\textbf{Initialization}: $\tilde{n}_0 \leftarrow 2$, $\tilde{\Theta}^{-1} \leftarrow \Theta$}
\STATE{\textbf{Foreach period} $k=0,1,\dots$ \textbf{do}}
\STATE{Initialize confidence sets: $\tilde{\Theta}_{0}^k \leftarrow \tilde{\Theta}^{k-1}$, $\tilde{\mathcal{A}}_{0}^k \leftarrow \mathcal{A}^*(\tilde{\Theta}_{0}^k)$}
\STATE{Run Algorithm \ref{alg:isucb} with $n = \tilde{n}_k$, $\tilde{\Theta}_0 = \tilde{\Theta}_{0}^k$, and $\tilde{\A}_0 = \tilde{\mathcal{A}}_{0}^k$}
\STATE{Update horizon: $\tilde{n}_{k+1} \leftarrow \tilde{n}_k^{1+\eta}$}
\STATE{\textbf{End}}
\end{algorithmic}
\end{algorithm}

Due to the additional complications introduced by the anytime extension (in particular, controlling the sets $\underline{\mathcal A}_h$), we were able to prove only a weaker bound than the one in Theorem \ref{th:regret-main} which, however, retains the same benefits. The proofs are reported in Appendix \ref{app:proofs4}.
\begin{restatable}[]{theorem}{thanytimemain}
\label{th:anytime-main}
Let $\eta = 1$, $\alpha = 2$, and $\beta = 1$. Then,
\begin{align*}
        R_n^{ASAE}(\theta^*,\Theta) \leq \sum_{i \in \mathcal{A}^*\setminus\{i^*\}} \frac{192 \Delta_{i}(\theta^*) \log n}{\Psi(\Theta_i^*,\{i,i^*\})} + 6|\mathcal{A}^*(\Theta)|.
\end{align*}
\end{restatable}
The new bound has the same form as the one of Algorithm \ref{alg:isucb}, except for the fact that the set of active arms for eliminating each $i$ is reduced to $\{i,i^*\} \subseteq \A_i^*$. Note, however, that the presence of these two arms is enough to prove Proposition~\ref{prop:comparison-UCB} and~\ref{prop:comparison-SUCB}.
\begin{remark}
Algorithm \ref{alg:isucb-anytime} is sub-UCB and, under the same conditions as in Proposition \ref{prop:comparison-SUCB}, is also sub-SUCB.
\end{remark}

We now prove a constant-regret bound for Algorithm \ref{alg:isucb-anytime}. We need the following assumption from \citep{lattimore2014bounded}, which was proven both necessary and sufficient to achieve constant regret.
\begin{assumption}[Informative optimal arm]\label{ass:dist} The structure $\Theta$ satisfies
\begin{align*}
\Gamma_* := \inf_{\theta \in \Theta \setminus \Theta_{i^*}^*} \Gamma_{i^*}(\theta,\theta^*) > 0.    
\end{align*}
\end{assumption}
In words, when a model is $\Gamma^*$-distant (or less) in arm $i^*$ from $\theta^*$, its optimal arm is still $i^*$. Therefore, pulling $i^*$ eventually discards all sub-optimal arms. This is fundamental to guarantee that, after the algorithm has pulled $i^*$ a sufficient number of times, no sub-optimal arm can become active again due to the increasing period length (hence we choose $i^*$ forever).

\begin{restatable}[]{theorem}{thanytimemainconst}
\label{th:anytime-main-const}
Let $\eta=1$, $\alpha = \frac{5}{2}$, $\beta=1$, $\bar{t} := \frac{20 |\mathcal{A}^*(\Theta)|\log 2}{\Gamma_*^2} + 2|\mathcal{A}^*(\Theta)|$, and suppose Assumption \ref{ass:dist} holds. Then,
\begin{align*}
           R_n^{ASAE}(\theta^*,\Theta) \leq \sum_{i \in \mathcal{A}^*\setminus\{i^*\}} \frac{480 \Delta_i(\theta^*)\log \bar{t}}{\Psi(\Theta_i^*,\{i,i^*\})} + 9|\mathcal{A}^*(\Theta)|.
\end{align*}
\end{restatable}
This bound improves over the one shown by \cite{lattimore2014bounded} for SUCB in its dependence on $\bar{t}$, which can be understood as the time at which the algorithm transitions to the constant regret regime. While \cite{lattimore2014bounded} proved $\bar{t} \simeq \mathcal{O}(\max\{1/\Gamma_*^2,1/\Delta_{\text{min}}^2\})$, here we show that such time does not depend on the minimum gap $\Delta_{\text{min}} = \min_{i : \Delta_i(\theta^*) > 0} \Delta_i(\theta^*)$. This is intuitive since, by Assumption \ref{ass:dist}, $\mathcal{O}(1/\Gamma_*^2)$ pulls of $i^*$ should be enough to identify the optimal arm. Although the analysis of SUCB can be improved by replacing the minimum sub-optimality gap with the minimum model gap, it seems that this dependence is tight. As an example, consider a structure in which the optimal arm is very informative ($\Gamma_* \gg 0$) but never optimistic. SUCB will never pull it until all optimistic models are discarded, which requires $\mathcal{O}(1/\Gamma^2_{\text{min}})$ steps in the worst case. Note that, whenever it is applied to structures satisfying Assumption \ref{ass:dist}, the bound of Theorem \ref{th:anytime-main} does not show constant regret since the proof uses an implicit worst-case argument (i.e., Assumption \ref{ass:dist} is assumed false).


\section{Constant-Regret Lower Bound}\label{sec:lb}

We have seen that SUCB and SAE are order-optimal for structures in $\Omega^{\text{wc}}$ and $\Omega^{\text{wc}} \cap \Omega^{\text{opt}}$, respectively. One might wonder whether we can still guarantee optimality in some structures where constant regret is achievable (i.e., when Assumption \ref{ass:dist} holds). We answer this question affirmatively by deriving a finite-time lower bound on the expected regret of any 'good' strategy. Note that the problem is non-trivial since, under Assumption \ref{ass:dist}, one cannot build hard models that differ from the true bandit only in the mean of one arm as in the proof of standard lower-bounds \cite[e.g.,][]{burnetas1996optimal}. Before stating our result, we specify the class of strategies under consideration. We shall use the following definition due to \cite{garivier2018explore}, which have been adopted to derive finite-time lower-bounds.
\begin{definition}[Super-fast convergence]
A strategy $\pi$ is super-fast convergent on a set $\Theta$ if there exists a constant $c>0$ such that, for any model $\theta\in\Theta$ and sub-optimal arm $i\in\mathcal{A}$, it satisfies
\begin{align*}
\E_{\theta}[T_i(n)] \leq \frac{c\log n}{\Delta_i(\theta)^2}.
\end{align*}
\end{definition}
It is easy to see that UCB, SUCB, and SAE are examples of super-fast convergent strategies. Furthermore, we call the class of structures considered in the lower bound \emph{worst-case constant regret} and define it as
\begin{align*}
\Omega^{\text{cr}} := &\{ \Theta \in \Omega\ |\ \forall \theta\in\Theta\setminus\Theta_{i^*}^* :\\ &\Gamma_{i^*}(\theta,\theta^*)=\Gamma_* \wedge \Gamma_j(\theta,\theta^*)=0\ \forall j\neq i^*(\theta),i^*\}.
\end{align*}
This can be understood as a generalization of the worst-case structure to make Assumption \ref{ass:dist} hold. Due to the challenges in deriving the lower bound for large $\Gamma_*$, we also need to assume that $0 < \Gamma^* \leq \mathcal{O}\left(\sqrt{\frac{1}{\sum_{i \neq i^*} \Delta_i^{-2}(\theta^*)}}\right)$, with the precise dependence given in Appendix \ref{app:lb}. Note that $\Gamma_*$ is a function of the structure and the dependence was omitted for conciseness.
We are now ready to state our result.
\begin{restatable}[]{theorem}{thlb}\label{th:lb}
Let $\Theta\in\Omega^{\text{cr}}$ and $n \geq \frac{1}{\Gamma_*^2}$. Then, for sufficiently small $\Gamma^*$, the expected regret of any super-fast convergent strategy $\pi$ can be lower bounded by
\begin{equation*}\label{eq:lb}
        R_n^\pi(\theta^*,\Theta) \geq \sum_{i \in \A^*\setminus\{i^*\}} \frac{\Delta_i(\theta^*)}{2\Psi(\Theta_i^*,\{i\})}\log \frac{\Delta^2}{4 e^2 c \Gamma_*^2 \log \frac{1}{\Gamma_*^2}},
\end{equation*}
where $\Delta := \inf_{\theta' \in \Theta \setminus \Theta_{i^*}^*}\Delta_{i^*}(\theta')$.
\end{restatable}
The proof, which combines ideas from \cite{garivier2018explore} and \cite{degenne2018bandits}, is reported in Appendix \ref{app:lb}. Note that the lower bound is positive for sufficiently small $\Gamma_*$. Apart from other constants, the dependence on $\Gamma_*$ matches the upper bound of Theorem \ref{th:anytime-main-const}. However, Theorem \ref{th:anytime-main-const} seems tighter due to the larger set of arms in $\Psi$ at the denominator. This is not surprising since the lower bound considers only structures with well-chosen hard models. It is easy to prove that, when SAE or SUCB are applied to structures in $\Omega^{\text{cr}}$, the two bounds match.

Other lower bounds for constant-regret settings have recently been derived. \cite{bubeck2013bounded} showed that, for the classic unstructured problems, it is enough to know $\mu^*$ and a lower bound on the minimum gap to achieve a constant regret.  \cite{garivier2018explore} refined this result by showing that the knowledge of $\mu^*$ alone actually suffices. \cite{lattimore2014bounded} studied several specific structured problems where constant regret is (or is not) possible, providing both lower bounds and algorithms to match them. Finally, we note that the asymptotic lower bound by \cite{combes2017minimal} is zero when Assumption \ref{ass:dist} holds as the regret scaled by $\log n$ correctly vanishes as $n$ grows. Their algorithm reduces to a greedy strategy in this setting which is not necessarily finite-time optimal according to Theorem \ref{th:lb}.



\section{Numerical Simulations}\label{sec:experiments}

\begin{figure*}[t!]
\centering
\hspace{-0.5cm}
\begin{subfigure}[t]{0.22\linewidth}
\centering
\includegraphics[height=3.5cm]{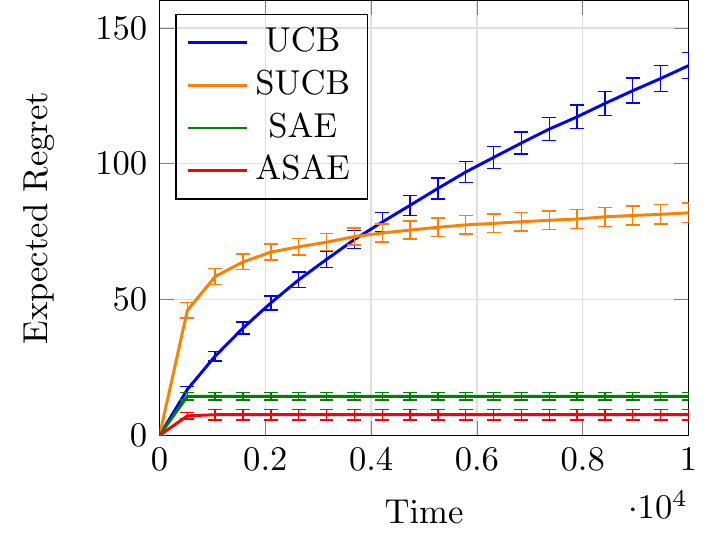}
\caption{}
\label{fig:pwlstruct}
\end{subfigure}
\hspace{0.5cm}
\begin{subfigure}[t]{0.22\linewidth}
\centering
\includegraphics[height=3.5cm]{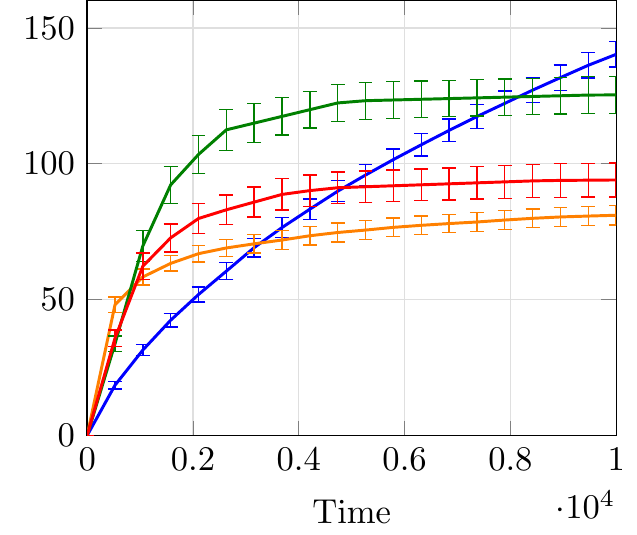}
\caption{}
\label{fig:pwlstructnoinfo}
\end{subfigure}
\hspace{0.01cm}
\begin{subfigure}[t]{0.22\linewidth}
\centering
\includegraphics[height=3.5cm]{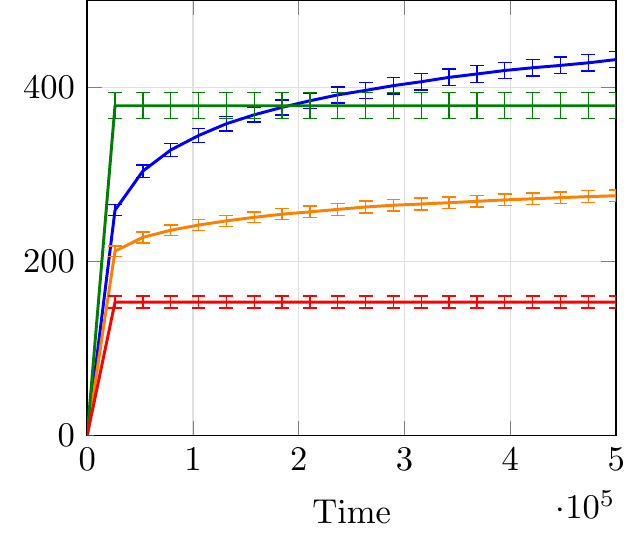}
\caption{}
\label{fig:specstruct}
\end{subfigure}
\hspace{0.2cm}
\begin{subfigure}[t]{0.22\linewidth}
\centering
\includegraphics[height=3.5cm]{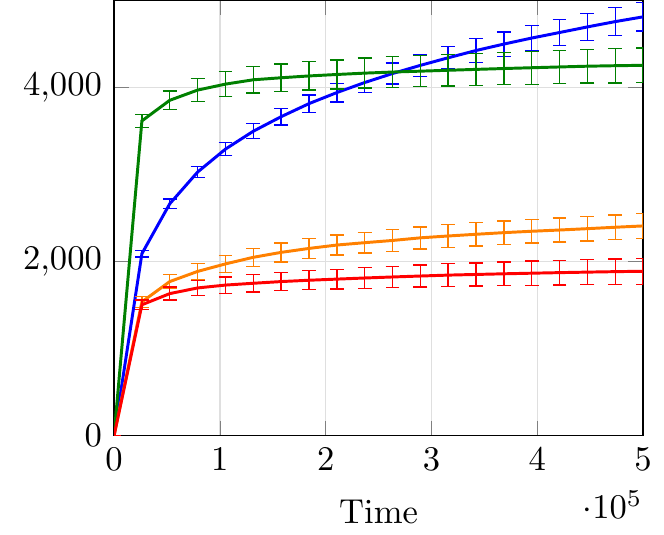}
\caption{}
\label{fig:randstruct}
\end{subfigure}
\caption{Expected regret in (a) the structure of Figure \ref{fig:example}\emph{(left)}, (b) the same structure with non-informative arm $2$, (c) the structure of Figure \ref{fig:example}\emph{(right)}, and (d) randomly-generated structures.}
\end{figure*}

We perform two different classes of experiments. In the first one, we consider well-chosen structures that allow us to better understand the behavior of all algorithms. In the second one, we randomize the structures to provide a more general comparison. In all experiments, we run SAE and its anytime version (ASAE), SUCB, and UCB on Bernoulli bandits. We also compared to the WAGP algorithm of \cite{atan2018global}, which however incurred linear regret in all our experiments (their assumptions never hold in our structures) and, therefore, is omitted from the plots. We use $\alpha=2$ for all algorithms and $\beta=1$ for SAE. Each plotted curve is the average of $100$ independent runs with $95\%$ Student's t confidence intervals.

\paragraph*{Hand-coded Structures}

We first consider the structure of Figure \ref{fig:example}\emph{(left)}. We set $n=10,000$ and $\eta=0.1$. The results are shown in Figure \ref{fig:pwlstruct}. SUCB suffers a large regret for removing models in which arm $3$ is optimal. On the other hand, SAE quickly discards these models by pulling arm $2$, which, in turn, is eliminated by pulling arm $1$. Hence the much lower regret, with the anytime version that performs slightly better. Notice also that Assumption \ref{ass:dist} is verified and SAE obtains constant regret. SUCB eventually transitions to constant regret too but needs a longer horizon. Alternatively, we can show an example where SUCB is expected to perform better. We modify the structure of Figure \ref{fig:example}\emph{(left)} to make arm $2$ non-informative (i.e., we set its mean to the highest value in the figure for all models) and run the experiment under the same setting. Figure \ref{fig:pwlstructnoinfo} shows that, as expected, SAE suffers from some additional regret for discarding the useless arm and performs worse than SUCB. However, it remains sub-UCB as proved in Section \ref{sec:discussion}.

We now consider the structure of Figure \ref{fig:example}\emph{(right)}. We set $n=500,000$, $\eta=0.01$, and report the results in Figure \ref{fig:specstruct}. The arm ordering induced by SUCB (from the most optimistic to the optimal one) leads the algorithm to discard arm 4 before even pulling it once.
Such arm, however, could be used to quickly discard arm 3, which is what SAE does. Notice that the larger regret of SAE with respect to its anytime counterpart is mainly due to the fact that phased procedures update the confidence sets much less than online approaches. This drawback is alleviated in the anytime version, which reduces the duration of some of these phases and retains good empirical performance.

\paragraph*{Randomized Structures}

We now consider random structures. In each run, we first randomize a set of $100$ models with $50$ arms by drawing their means from the uniform distribution and we randomly choose the true model among them. Then, we build $50$ additional 'hard' models by perturbing a random arm of the true model to become optimal and optimistic, and another random arm to become informative. In particular, the mean of the first random arm is set to $\mu^*(\theta^*) + 0.2\epsilon$, with $\epsilon \sim \mathcal{U}([0,1])$, while the second to $1/10$ of the original mean (so that we potentially get a larger model gap). The results are shown in Figure \ref{fig:randstruct}. Most of the regret suffered by SUCB is due to the hard instances we introduced. Some of them are likely to be eliminated by informative arms, but this is not always guaranteed by the SUCB strategy. Both versions of SAE, on the other hand, implicitly exploit these informative arms, with the anytime version outperforming all alternatives. Once again, the original version suffers a high initial regret due to the phased procedure.


\section{Discussion}

Similarly to most of related literature, our SAE algorithm confirms that simple confidence-based strategies can be designed to exploit general structures, though so far they have been proven optimal only for worst-case structures.
Although it only pulls potentially-optimal arms, SAE is not optimistic. The design of non-optimistic algorithms is a key step towards optimality since it is known that OFU-based strategies are not optimal for general structures \citep{lattimore2017end,combes2017minimal, hao2019adaptive}. Our regret bounds fully reflect the structure-awareness and their derivation might be of independent interest for analyzing other approaches. Although considering phased strategies is one of our key choices to both obtain the desired algorithmic properties and simplify the proofs, we show empirically that SAE does not suffer from it too much. In particular, it outperforms online strategies in specific structures where informative arms exist that are not always pulled with the OFU principle.

The key open question is how to design confidence-based strategies that are optimal for general structures. The algorithms discussed in this paper have been proven optimal only for certain worst-case structures, while algorithms like OSSB are asymptotically optimal for general structures but require to force exploration to solve an oracle optimization problem. Whether the optimal pull counts of a lower-bound like the one by \cite{combes2017minimal} can be attained in confidence-based settings and with good finite-time performance remains unknown. We believe that recent advances in the context of pure exploration for bandit problems \citep{menard2019gradient, degenne2019non} might provide useful insights into this problem. Furthermore, a finite-time extension of the asymptotic lower bound for general structures, and the corresponding design of finite-time optimal algorithms,  is a challenging but interesting research direction.


\bibliographystyle{apalike}
\bibliography{biblio.bib}

\newpage
\appendix
\onecolumn


\section{Notation}\label{app:symbols}

\begin{table*}[h]
\centering
\begin{tabular}{@{}cc@{}} 
\toprule
Symbol & Meaning \\
\cmidrule{1-2}
$\Thetall$ & Set of all bandit problems\\
$\mathcal{A}$ & Set of arms\\
$\Theta$ & The structure (a subset of $\Thetall$) available to the algorithm\\
$\theta^*$ & The true model\\
$n$ & The learning horizon\\
$\nu_i(\theta)$ & The distribution of arm $i$ of model $\theta$\\
$\mu_i(\theta)$ & The mean of arm $i$ of model $\theta$\\
$\mu^*(\theta)$ & The optimal mean of model $\theta$\\
$i^*(\theta)$ & The (unique) optimal arm of model $\theta$\\
$\Delta_i(\theta)$ & The sub-optimality gap of arm $i$ in model $\theta$\\
$\Gamma_i(\theta,\theta')$ & The model gap of arm $i$ between models $\theta$ and $\theta'$\\
$\Psi(\Theta',\A')$ & The maximum (over arms in $\A'$) model gap between $\theta^*$ and the most similar model $\theta\in\Theta'$\\
$\psi(\Theta',\A')$ & The hardest model in $\Theta'$ using arms in $\A'$\\
$\A^*(\Theta)$ & Set of arms which are optimal for at least one model in $\Theta$\\
$\Theta_i^*$ & Set of models with $i$ as optimal arm\\
$\Theta_i^+$ & Set of optimistic models w.r.t. $\theta^*$ with $i$ as optimal arm\\
$R_n^\pi(\theta,\Theta)$ & Expected regret of strategy $\pi$ in bandit $\theta$ under structure $\Theta$\\
$\tilde{\Theta}_h$ & Confidence set in phase $h$\\
$\tilde{\A}_h$ & Active arms in phase $h$\\
$T_i(h)$ & Number of pulls of arm $i$ at the end of phase $h$\\
$\hat{\mu}_{i,h}$ & Empirical mean of arm $i$ at the end of phase $h$\\
$\bar{\A}_h$ & Set of arms which are, with high probability, discarded no later than phase $h$\\
$\underline{\A}_h$ & Set of arms which are, with high probability, active in phase $h$\\
${\A}_h$ & Set of arms which are, with high probability, potentially active in phase $h$\\
$\bar{h}_i$ & The last phase at which $i$ is, with high probability, potentially active\\
$\A_i^*$ & Set of arms which are, with high probability, active for discarding $i$\\
$\Gamma_*$ & Minimum model gap of $i^*$ between the true model and any other with a different optimal arm\\
$\tilde{\Theta}_h^k$ & Confidence set in phase $h$ of period $k$\\
$\tilde{\A}_h^k$ & Active arms in phase $h$ of period $k$\\
$T_i(k,h)$ & Number of pulls of arm $i$ at the end of phase $h$ of period $k$\\
$\hat{\mu}_{i,h}^k$ & Empirical mean of arm $i$ at the end of phase $h$ of period $k$\\
$\Omega^{\text{gen}}$ & General structure (all sets containing $\theta^*$)\\
$\Omega^{\text{wc}}$ & Worst-case structure\\
$\Omega^{\text{opt}}$ & Optimistic structure\\
$\Omega^{\text{cr}}$ & Worst-case constant-regret structure\\
$\Omega^{\text{conf}}$ & Confusing structure\\
\bottomrule
\end{tabular}
\caption{The notation adopted in this paper.}
\label{tab:notation}
\end{table*}

\newpage

\section{Proof of Theorem \ref{th:sucb}}\label{app:sucb-proof}

We analyze the SUCB version of \cite{lattimore2014bounded} (called UCB-S by the authors) using ideas from \cite{azar2013sequential}. We recall that, at each time step $t$, the algorithm builds a confidence set
\begin{align*}
\tilde{\Theta}_t = \left\{ \theta \in \Theta\ |\ \forall i\in A : |\mu_i(\theta) - \hat{\mu}_{i,t}| < \sqrt{\frac{2\alpha\sigma^2\log t}{T_i(t-1)}} \right\},
\end{align*}
where the distribution of each arm is assumed sub-Gaussian with variance factor $\sigma^2$. Then, the algorithm pulls the optimistic arm according to the models in this set,
\begin{align*}
I_t \leftarrow \argmax_{i \in \A}\sup_{\theta\in\tilde{\Theta}_t}\mu_i(\theta).
\end{align*}
The regret bound proved by \cite{lattimore2014bounded} (see their Theorem 2) has the same form as the one of UCB. That is, for a suitable choice of $\alpha$, there exist constants $c,c'$ such that
\begin{align*}
R_n^{\text{SUCB}}(\theta^*, \Theta) \leq \sum_{i \neq i^*} \frac{c\log n}{\Delta_i(\theta^*)} + c'.
\end{align*}
This bound, however, does not fully reflect how the algorithm exploits the given structures. The bound in Theorem 1 of \cite{azar2013sequential}, on the other hand, has the same form as the one we prove here, but it holds only for a finite set of models, while the one of \cite{lattimore2014bounded} does not have such restriction. We now prove Theorem \ref{th:sucb}, which straightforwardly combines the analyses of these two papers, thus providing a regret bound that scales with the model gaps rather than the sub-optimality gaps and that holds for any structure.

\thsucb*

\begin{proof}
Let $F_t := \indi{ \theta^* \in \tilde{\Theta}_t}$. Consider any sub-optimal arm $i$ and suppose $I_t = i$ and $F_t = 1$. Since $i$ is pulled, there exists some $\bar{\theta} \in \tilde{\Theta}_t$ such that $\bar{\theta} \in \Theta_{i}^+$. These facts imply
\begin{equation}
    \Gamma_i(\bar{\theta}, \theta^*) = |\mu_i(\bar{\theta}) - \mu_i(\theta^*)| \leq |\mu_i(\bar{\theta}) - \hat{\mu}_{i,t}| + |\hat{\mu}_{i,t} - \mu_i(\theta^*)| \leq 2\sqrt{\frac{2\alpha\sigma^2\log t}{T_i(t-1)}}.
\end{equation}
Therefore,
\begin{equation*}
    T_i(t-1) \leq \frac{8\alpha\sigma^2\log t}{\Gamma_i^2(\bar{\theta}, \theta^*)} \leq \left\lceil\frac{8\alpha\sigma^2\log n}{\inf_{\theta \in \Theta_{i}^+}\Gamma_i^2(\theta, \theta^*)}\right\rceil =: u_i(n).
\end{equation*}
Then,
\begin{align*}
    \E[T_i(n)] &= \E\left[ \sum_{t=1}^n \mathbb{1}\{I_t = i\} \right] = \E\left[ \sum_{t=1}^n \mathbb{1}\{I_t = i \wedge T_i(t) \leq u_i(n) \} \right] + \E\left[ \sum_{t=1}^n \mathbb{1}\{I_t = i \wedge T_i(t) > u_i(n) \} \right]\\ &\leq u_i(n) + \E\left[ \sum_{t=u_i(n)+1}^n \mathbb{1}\{I_t = i \wedge T_i(t) > u_i(n) \} \right] \leq u_i(n) + \E\left[ \sum_{t=u_i(n)+1}^n \mathbb{1}\{I_t = i \wedge F_t = 0 \} \right],
\end{align*}
where the last inequality follows since pulling arm $i$ at time step $t$ implies that either $T_i(t) \leq u_i(n)$ or the true parameter is not in the confidence set (i.e., $F_t=0$). Then,
\begin{align*}
    R_n & \stackrel{(a)}{=} \sum_{i \in \mathcal{A}^*(\Theta)} \Delta_i(\theta^*) \E[T_i(n)] \stackrel{(b)}{\leq} \sum_{i \in \mathcal{A}^*(\Theta)} \Delta_i(\theta^*) \left( u_i(n) + \E\left[ \sum_{t=u_i(n)+1}^n \mathbb{1}\{I_t = i \wedge F_t = 0 \} \right] \right)\\ & \stackrel{(c)}{\leq} \sum_{i \in \mathcal{A}^*(\Theta)} \Delta_i(\theta^*) u_i(n) + \Delta_{\text{max}}\sum_{t=1}^n\mathbb{P}\{F_t=0\}.
\end{align*}
where (a) holds since arms that are sub-optimal for all models in $\Theta$ are never pulled, (b) follows from the bound on the number of pulls derived above, and (c) follows from the definition of $\Delta_{\text{max}} = \max_{i \in \A^*(\Theta)} \Delta_i(\theta^*)$ and the fact that at each time only one arm is pulled. The second term can be bounded using Lemma 5 of \cite{lattimore2014bounded} (by taking the union bound only over $\A^*(\Theta)$) by
\begin{equation*}
    \sum_{t=1}^n\mathbb{P}\{F_t=0\} \leq 2|\A^*(\Theta)|\sum_{t=1}^n t^{1-\alpha} \leq \frac{2|\A^*(\Theta)|(\alpha - 1)}{\alpha - 2}.
\end{equation*}
The theorem follows by combining the last two displays and renaming the constants.
\end{proof}

\section{Proofs of Section \ref{sec:sae}}\label{app:proofs3}

\subsection{Proof of Theorem \ref{th:regret-main}}

We begin by showing that, with high probability, the true model $\theta^*$ is always contained in the confidence set by a certain margin (which depends on $\beta$). Unlike previous works, we need this to guarantee that sub-optimal arms are not eliminated too early.

\begin{restatable}[]{lemma}{lemmaconf}\label{lemma:conf}
Let $\alpha >0$, $\beta \geq 1$, and $E = \{\forall h=0,\dots,\lceil \log_2 n \rceil : E_h\ \text{holds}\}$, with $E_h$ denoting the following event:
\begin{equation*}
E_h := \left\{ \forall i \in \mathcal{A} : |\hat{\mu}_{i,h-1} - \mu_i(\theta^*)| < \frac{1}{\beta} \sqrt{\frac{\alpha \log n}{T_i(h-1)}}\right\}.
\end{equation*}
Then, the probability that $E$ does not hold can be upper bounded by
\begin{equation*}
\prob{E^c} \leq |\mathcal{A}^*(\Theta)| n^{-2\frac{\alpha}{\beta^2}} (\log_2 n + 2)^2.
\end{equation*}
\end{restatable}

\begin{proof}
Using the union bound, we have
\begin{align*}
\prob{E^c} &= \prob{\exists h=1,\dots,\lceil \log_2 n \rceil, \exists i \in \mathcal{A} : |\hat{\mu}_{i,h-1} - \mu_i(\theta^*)| \geq \frac{1}{\beta} \sqrt{\frac{\alpha \log n}{T_i(h-1)}} \wedge T_i(h-1) > 0}\\ &\leq \sum_{h=1}^{\lceil \log_2 n \rceil} \sum_{i \in \mathcal{A}^*(\Theta)} \prob{|\hat{\mu}_{i,h-1} - \mu_i(\theta^*)| \geq \frac{1}{\beta} \sqrt{\frac{\alpha \log n}{T_i(h-1)}} \wedge T_i(h-1) > 0},
\end{align*}
where the sum starts from $h=1$ since in phase $0$ no arm has been pulled and all models are therefore contained in the confidence set. Furthermore, $\mathcal{A}$ can be replaced by $\mathcal{A}^*(\Theta)$ since arms that are sub-optimal for all models are never pulled and so the corresponding event above never holds. Let us now consider the inner term for a fixed phase $h$ and arm $i$. Notice that, at the end of phase $h-1$, the possible number of pulls of arm $i$ are
\begin{align*}
    k_s := \left\lceil\frac{\alpha\log n}{\tilde{\Gamma}_s^2}\left(1 + \frac{1}{\beta}\right)^2\right\rceil
\end{align*}
for $s = 0,1,\dots,h-1$. Thus, by taking a further union bound on the possible values of $T_i(h-1)$ and using Chernoff-Hoeffding inequality, we obtain
\begin{align*}
\prob{|\hat{\mu}_{i,h-1} - \mu_i(\theta^*)| \geq \frac{1}{\beta} \sqrt{\frac{\alpha \log n}{T_i(h-1)}}} &= \prob{\bigcup_{s=0}^{h-1} |\hat{\mu}_{i,h-1} - \mu_i(\theta^*)| \geq \frac{1}{\beta} \sqrt{\frac{\alpha \log n}{T_i(h-1)}} \wedge T_i(h-1) = k_s}\\ &\leq \sum_{s=0}^{h-1} \prob{|\hat{\mu}_{i,k_s} - \mu_i(\theta^*)| \geq \frac{1}{\beta} \sqrt{\frac{\alpha \log n}{k_s}}}\\ &\leq \sum_{s=0}^{h-1} 2e^{-2k_s\frac{\alpha\log n}{\beta^2 k_s}} = 2 h n^{-2\frac{\alpha}{\beta^2}}.
\end{align*}
Notice that, with some abuse of notation, we define $\hat{\mu}_{i,k_s}$ as the empirical mean of arm $i$ after $k_s$ pulls of such arm. Putting everything together,
\begin{align*}
\prob{E^c} &\leq \sum_{h=1}^{\lceil \log_2 n \rceil} \sum_{i \in \mathcal{A}^*(\Theta)} 2hn^{-2\frac{\alpha}{\beta^2}} = 2 |\mathcal{A}^*(\Theta)| n^{-2\frac{\alpha}{\beta^2}} \sum_{h=1}^{\lceil \log_2 n \rceil} h \leq |\mathcal{A}^*(\Theta)| n^{-2\frac{\alpha}{\beta^2}} (\log_2 n + 2)^2,
\end{align*}
which concludes the proof.
\end{proof}

Next, we show a sufficient condition for eliminating a model from the confidence set.

\begin{lemma}\label{lemma:model-elim}
Suppose there exists an arm $i \in \mathcal{A}$, a model $\theta\in\Theta$, and a phase $h \geq 0$ such that $T_i(h) \geq \left(1 + \frac{1}{\beta}\right)^2 \frac{\alpha \log n}{\Gamma_i^2(\theta, \theta^*)}$. Then, under event $E$, $\theta \notin \tilde{\Theta}_{h'}$ for all $h' > h$.
\end{lemma}
\begin{proof}
Suppose there exists a phase $h' > h$ such that $\theta \in \tilde{\Theta}_{h'}$. Then,
\begin{align*}
\Gamma_i(\theta, \theta^*) &= |\mu_i(\theta) - \mu_i(\theta^*)| \stackrel{(a)}{\leq} |\mu_i(\theta) - \hat{\mu}_{i,h'}| + |\hat{\mu}_{i,h'} - \mu_i(\theta^*)|\\ &\stackrel{(b)}{<} \left(1 + \frac{1}{\beta}\right)\sqrt{\frac{\alpha\log n}{T_i(h'-1)}} \stackrel{(c)}{\leq} \left(1 + \frac{1}{\beta}\right)\sqrt{\frac{\alpha\log n}{T_i(h)}},
\end{align*}
where (a) follows from the triangle inequality, (b) from the fact that $\theta$ is in the confidence set and $E$ holds, and (c) from $h'>h$ and the monotonicity of the number of pulls. Therefore, it must be that
\begin{align*}
T_i(h) < \left(1 + \frac{1}{\beta}\right)^2 \frac{\alpha \log n}{\Gamma_i^2(\theta, \theta^*)},
\end{align*}
which is a contradiction. Thus, we must have $\theta \notin \tilde{\Theta}_{h'}$.
\end{proof}

We now show a condition on the number of pulls such that, under the 'good' event $E$, an arm is discarded.

\begin{restatable}[]{lemma}{lemmaarmel}\label{lemma:arm-el}
Let $h \geq 0$, $i \in \mathcal{A}$, and suppose that, for any model $\theta \in \Theta^*_i$ there exists an arm $j \in \mathcal{A}$  such that $T_j(h) \geq \left(1 + \frac{1}{\beta}\right)^2 \frac{\alpha \log n}{\Gamma_j^2(\theta, \theta^*)}$. Then, under event $E$, $i \notin \tilde{\mathcal{A}}_{h'}$ for all $h' > h$.
\end{restatable}

\begin{proof}
All models with $i$ as optimal arm are discarded in phase $h$ by Lemma \ref{lemma:model-elim}. Therefore, $\forall \theta \in \Theta^*_i : \theta \notin \tilde{\Theta}_{h+1}$, which also implies that $i \notin \tilde{\mathcal{A}}_{h'}$ for all $h' > h$.
\end{proof}

Next, we show that, when all arms have not been pulled too much, some models can be guaranteed to lie in the confidence set.

\begin{lemma}\label{lemma:isucb-model-in}
Let $h \geq 0$, $\theta \in \Theta$, and suppose $T_i(h) \leq \left(1 - \frac{1}{\beta}\right)^2 \frac{\alpha \log n}{\Gamma_i^2(\theta, \theta^*)}$ for all arms $i \in \mathcal{A}$. Then, under event $E$, $\theta \in \tilde{\Theta}_{h+1}$.
\end{lemma}
\begin{proof}
Notice that, for all arms $i \in \mathcal{A}$, $\Gamma_i(\theta, \theta^*) \leq \left(1 - \frac{1}{\beta}\right)\sqrt{\frac{\alpha\log n}{T_i(h)}}$. Therefore,
\begin{align*}
|\hat{\mu}_{i,h} - \mu_i(\theta)| &\stackrel{(a)}{\leq} |\hat{\mu}_{i,h} - \mu_i(\theta^*)| + |\mu_i(\theta^*) - \mu_i(\theta)| = |\hat{\mu}_{i,h} - \mu_i(\theta^*)| + \Gamma_i(\theta, \theta^*)\\ &\stackrel{(b)}{<} \frac{1}{\beta}\sqrt{\frac{\alpha\log n}{T_i(h)}} + \Gamma_i(\theta, \theta^*) \stackrel{(c)}{\leq} \sqrt{\frac{\alpha\log n}{T_i(h)}},
\end{align*}
where (a) follows from the triangle inequality, (b) from the fact that $E$ holds, and (c) from the condition on the number of pulls above. This implies that $\theta\in\tilde{\Theta}_{h+1}$.
\end{proof}

The following lemma states a condition on $\tilde{\Gamma}_{h-1}$ under which a model $\theta\neq\theta^*$ can be guaranteed to belong to $\tilde{\Theta}_h$.

\begin{restatable}[]{lemma}{lemmamodelinphase}\label{lemma:model-inphase}
Let $h \geq 1$, $\theta \in \Theta$, and $\alpha \geq \beta^2$. For all $i \in \mathcal{A}^*(\Theta)$, let $\tilde{h}_i \leq h - 1$ be such that either $i \notin \tilde{A}_{\tilde{h}_i + 1}$ or $\tilde{h}_i = h-1$. Suppose the following condition holds
\begin{align}\label{eq:cond-gamma}
    \tilde{\Gamma}_{h-1} \geq k_\beta \max_{j\in\mathcal{A}^*(\Theta)}\frac{\Gamma_j(\theta,\theta^*)}{2^{h - \tilde{h}_j - 1}}.
\end{align}
Then, under event $E$, $\theta \in \tilde{\Theta}_h$.
\end{restatable}

\begin{proof}
Fix any arm $i \in \mathcal{A}^*(\Theta)$. By assumption $i$ is pulled at most in phase $\tilde{h}_i$. Therefore, its number of pulls at the end of phase $h-1$ can be bounded by
\begin{align*}
T_i(h-1) = \left\lceil\frac{\alpha\log n}{\tilde{\Gamma}_{\tilde{h}_i}^2}\left(1 + \frac{1}{\beta}\right)^2\right\rceil = \left\lceil\frac{\alpha\log n}{4^{h - \tilde{h}_i - 1}\tilde{\Gamma}_{h-1}^2}\left(1 + \frac{1}{\beta}\right)^2\right\rceil \leq \frac{\alpha\log n}{4^{h - \tilde{h}_i - 1}\tilde{\Gamma}_{h-1}^2}\left(1 + \frac{1}{\beta}\right)^2 + 1,
\end{align*}
where the second equality is from $\tilde{\Gamma}_{\tilde{h}_i} = \frac{1}{2^{\tilde{h}_i}} = \frac{2^{h-1}}{2^{\tilde{h}_i}2^{h-1}} = 2^{h- \tilde{h}_i - 1}\tilde{\Gamma}_{h-1}$.
The constant term can be upper bounded by
\begin{align*}
1 = \frac{(\beta+1)^2\log n}{(\beta+1)^2\log n} \stackrel{(a)}{\leq} \frac{\alpha(\beta+1)^2\log n}{\beta^2(\beta+1)^2\log n}4^{\tilde{h}_i} \frac{4^{h-1}}{4^{h-1}} \stackrel{(b)}{=} \frac{1}{(\beta+1)^2\log n}\frac{\alpha\log n}{4^{h - \tilde{h}_i - 1}\tilde{\Gamma}_{h-1}^2}\left(1+\frac{1}{\beta}\right)^2,
\end{align*}
where (a) follows from $\alpha \geq \beta^2$ and (b) from the definition of $\tilde{\Gamma}_{h-1}$. Hence,
\begin{align*}
T_i(h-1) &\stackrel{(a)}{\leq} \left(1 + \frac{1}{(\beta+1)^2\log n}\right)\frac{\alpha\log n}{4^{h - \tilde{h}_i - 1}\tilde{\Gamma}_{h-1}^2}\left(1 + \frac{1}{\beta}\right)^2\\ &\stackrel{(b)}{\leq} \left(1 - \frac{1}{\beta}\right)^2 \frac{\alpha \log n}{4^{h - \tilde{h}_i - 1}\max_{j\in\mathcal{A}^*(\Theta)}\frac{\Gamma_j^2(\theta,\theta^*)}{4^{h - \tilde{h}_j - 1}}} \leq \left(1 - \frac{1}{\beta}\right)^2 \frac{\alpha \log n}{\Gamma_i^2(\theta,\theta^*)},
\end{align*}
where in (a) we applied the two inequalities derived above and in (b) we used the condition \eqref{eq:cond-gamma} on $\tilde{\Gamma}_{h-1}$. This argument can be repeated for all other arms in $\mathcal{A}^*(\Theta)$. Therefore, Lemma \ref{lemma:isucb-model-in} together with the fact that arms not in $\mathcal{A}^*(\Theta)$ are never pulled, implies $\theta \in \tilde{\Theta}_h$.
\end{proof}

The following theorem is the key result that will be used to prove the final regret bound. It shows that the sets $\bar{\A}_h$ and $\underline{\A}_h$ defined in Section \ref{sec:sae} have the intended meaning.

\begin{restatable}[]{theorem}{tharmremoval}\label{th:arm-removal}
Let $\beta \geq 1$ and $\alpha = \beta^2$. Then, under event $E$, the following two statements are true for all $h \geq 0$:
\begin{align}\label{eq:armout}
&\forall i \in \bar{\mathcal{A}}_h : i \notin \tilde{\mathcal{A}}_{h'}\ \forall h'>h,\\
&\forall i \in \underline{\mathcal{A}}_h : i \in \tilde{\mathcal{A}}_{h}.\label{eq:armin}
\end{align}
\end{restatable}

\begin{proof}
We prove the theorem by induction on $h$.

\paragraph{1) Base case ($h = 0,1$)}

We show both $h=0$ and $h=1$ as base cases since the recursive definition of the sets $\underline{\mathcal{A}}_h$ starts from $h=1$ and depends on $\bar{\mathcal{A}}_h$. The recursive definition of the latter, on the other hand, starts from $h=0$.

\paragraph{1.1) First phase ($h=0$)}

Since $\tilde{A}_0 = \mathcal{A}^*(\Theta)$ by the initialization step of Algorithm \ref{alg:isucb}, \eqref{eq:armin} trivially holds. If $\bar{\mathcal{A}}_0$ is empty, \eqref{eq:armout} trivially holds as well. Suppose $\bar{\mathcal{A}}_0$ is not empty and fix any arm $i \in \bar{\mathcal{A}}_0$. For all arms $j \in \mathcal{A}^*(\Theta)$,
\begin{align*}
T_{j}(0) &\stackrel{(a)}{=} \left\lceil\frac{\alpha\log n}{\tilde{\Gamma}_{0}^2}\left(1 + \frac{1}{\beta}\right)^2\right\rceil
\\ &\stackrel{(b)}{\geq} \left\lceil\frac{\alpha\log n}{\inf_{\theta\in\Theta^*_i} \max_{l \in \mathcal{A}^*(\Theta)} \Gamma_l^2(\theta, \theta^*)}\left(1 + \frac{1}{\beta}\right)^2\right\rceil,
\end{align*}
where (a) is from the number of pulls in Algorithm \ref{alg:isucb} and the fact that all arms in $ \mathcal{A}^*(\Theta)$ are active, and (b) follows from the definition of $\bar{\mathcal{A}}_0$. Therefore, for all $\theta \in \Theta^*_i$ there exists some arm $j\in \mathcal{A}^*(\Theta)$ whose number of pulls at the end of phase $0$ is at least
\begin{align*}
    T_j(0) \geq \left\lceil\frac{\alpha\log n}{\Gamma_j^2(\theta, \theta^*)}\left(1 + \frac{1}{\beta}\right)^2\right\rceil.
\end{align*}
Hence, Lemma \ref{lemma:arm-el} ensures that $i \notin \tilde{\mathcal{A}}_{h'}$ for all $h'>0$, which in turn implies that \eqref{eq:armout} holds.

\paragraph{1.2) Second phase ($h=1$)}

Let us start from \eqref{eq:armin}. Take any arm $i \in \mathcal{A}_1 := \mathcal{A}^*(\Theta) \setminus \bar{\mathcal{A}}_0$ and suppose
\begin{align}\label{eq:lb-gamma}
    \tilde{\Gamma}_{0} > k_{\beta}\inf_{\theta\in\Theta^*_i} \max_{j \in \mathcal{A}^*(\Theta)} \frac{\Gamma_j(\theta, \theta^*)}{2^{\max\{-\bar{h}_j,0\}}}
\end{align}
holds. Since $2^{\max\{-\bar{h}_j,0\}} = 1$ for all $j \in \mathcal{A}^*(\Theta)$, \eqref{eq:lb-gamma} implies that there exists some model $\bar{\theta} \in \Theta^*_i$ such that $\tilde{\Gamma}_{0} \geq k_{\beta} \max_{j \in \mathcal{A}^*(\Theta)} \Gamma_j(\bar{\theta}, \theta^*)$.
 Thus, we can directly apply Lemma \ref{lemma:model-inphase} using $\tilde{h}_j = 0$ for all $j \in \mathcal{A}^*(\Theta)$ and obtain $\bar{\theta} \in \tilde{\Theta}_1$. This implies $i \in \tilde{\mathcal{A}}_1$, from which \eqref{eq:armin} holds.

The proof of \eqref{eq:armout} proceeds similarly as for $h=0$. Take any arm $i \in \bar{\mathcal{A}}_1$ (assuming the set is not empty). We have just proved that all arms $j \in \underline{\mathcal{A}}_1$ are pulled in phase $h=1$. If arm $i$ has already been removed, \eqref{eq:armout} trivially holds. Hence, we can safely assume that $i\in\tilde{\mathcal{A}}_1$. Therefore, arms in $\underline{\A}_1 \cup \{i\}$ are active and the number of pulls is sufficient to apply Lemma \ref{lemma:arm-el}, which implies \eqref{eq:armout}.

\paragraph{2) Inductive step ($h > 1$)}

Now assume the two statements hold for $h'=0,1,\dots,h-1$. This implies, in particular, that an arm $i\in\bar{\mathcal{A}}_{h'}$, $h'\leq h-1$, is not pulled after $h'$. Once again, take any arm $i \in \underline{A}_h$.
The definition of $\underline{A}_h$ implies
\begin{align*}
    \tilde{\Gamma}_{h-1} \geq k_{\beta}\max_{j \in \mathcal{A}^*(\Theta)} \frac{\Gamma_j(\bar{\theta}, \theta^*)}{2^{\max\{h-\bar{h}_j-1,0\}}}
\end{align*}
for some $\bar{\theta} \in \Theta^*_i$. Notice that, by the inductive assumption, all arms $j \in \mathcal{A}^*(\Theta) \setminus \mathcal{A}_h$ are not pulled after $\bar{h}_j \leq h-1$. On the other hand, for all arms $j \in \mathcal{A}_h$, it must be that $\bar{h}_j \geq h$. Thus, we can apply Lemma \ref{lemma:model-inphase} by setting $\tilde{h}_j = \bar{h}_j$ for arms $j \in \mathcal{A}^*(\Theta) \setminus \mathcal{A}_h$ and $\tilde{h}_j = h-1$ for arms $j \in \mathcal{A}_h$. Hence, $\bar{\theta} \in \tilde{\Theta}_h$ and \eqref{eq:armin} holds.

Finally, since all arms in $\underline{\mathcal{A}}_h$ are pulled in phase $h$, we can show that \eqref{eq:armout} holds using exactly the same argument as for the second base case ($h=1$).

\end{proof}

We are now ready to prove Theorem \ref{th:regret-main}.

\begin{proof}{\emph{(Theorem \ref{th:regret-main})}}
The expected regret can be written as
\begin{align*}
R_n &\stackrel{(a)}{\leq} \sum_{t=1}^n \expec{\Delta_{I_t}(\theta^*) | E} + n\prob{E^c} \stackrel{(b)}{=} \sum_{i \in \mathcal{A}} \Delta_i(\theta^*) \expec{T_i(n) | E} + n\prob{E^c},
\end{align*}
where in (a) we upper bounded the gaps by $1$ and used $\expec{\indi{E^c}} = \prob{E^C}$, while in (b) we used the standard rewriting in terms of the number of pulls.

We now upper bound the expected number of pulls of each sub-optimal arm $i$ when conditioned on event $E$. Since $i \in \bar{\mathcal{A}}_{\bar{h}_i}$, Theorem \ref{th:arm-removal} ensures that arm $i$ is not pulled after phase $\bar{h}_i$. Hence,
\begin{align*}
T_i(n) &\stackrel{(a)}{\leq} \left\lceil\frac{\alpha\log n}{\tilde{\Gamma}_{\bar{h}_i}^2}\left(1 + \frac{1}{\beta}\right)^2\right\rceil \stackrel{(b)}{=} \left\lceil\frac{4\alpha\log n}{\tilde{\Gamma}_{\bar{h}_i - 1}^2}\left(1 + \frac{1}{\beta}\right)^2\right\rceil\\ &\stackrel{(c)}{\leq} \left\lceil\frac{4(1+\beta^2)\log n}{\inf_{\theta\in\Theta^*_i} \max_{j \in \A_i^*} \Gamma_j^2(\theta, \theta^*)}\right\rceil,
\end{align*}
where (a) follows immediately from Theorem \ref{th:arm-removal} and Algorithm \ref{alg:isucb}, while (b) from $\tilde{\Gamma}_{\bar{h}_i} = \frac{\tilde{\Gamma}_{\bar{h}_i-1}}{2}$. To show (c), notice that $\tilde{\Gamma}_{\bar{h}_i - 1} > \inf_{\theta\in\Theta^*_i} \max_{j \in \A_i^*} \Gamma_j(\theta, \theta^*)$ from the definition of $\bar{\mathcal{A}}_{\bar{h}_i}$ (if this did not hold, arm $i$ would be eliminated in phase $\bar{h}_i-1$ since $\underline{\mathcal{A}}_{\bar{h}_i} \subseteq \underline{\mathcal{A}}_{\bar{h}_i-1}$). Therefore, the regret conditioned on event $E$ can be upper bound by
\begin{align*}
\sum_{i \in \mathcal{A}^*(\Theta)} \frac{4(1+\beta^2)\Delta_i(\theta^*)\log n}{\inf_{\theta\in\Theta^*_i} \max_{j \in \A_i^*} \Gamma_j^2(\theta, \theta^*)} + |\mathcal{A}^*(\Theta)|,
\end{align*}
where we used $\lceil x \rceil \leq x + 1$ and $\sum_{i \in \mathcal{A}^*(\Theta)} \Delta_{i}(\theta^*) \leq |\mathcal{A}^*(\Theta)|$.

Let us now consider the probability of $E$ not holding. Using Lemma \ref{lemma:conf} with $\alpha = \beta^2$, together with $(\log_2 n + 2)^2 \leq n$ for $n\geq 64$, we obtain
\begin{align*}
n\prob{E^c} \leq |\mathcal{A}^*(\Theta)|\frac{(\log_2 n + 2)^2}{n} \leq |\mathcal{A}^*(\Theta)|,
\end{align*}
which, combined with the previous bound, concludes the proof.

\end{proof}

\subsection{Proof of Proposition \ref{prop:comparison-UCB}}

\propcomparisonucb*

\begin{proof}
First notice that each sub-optimal arm $i$ is also in set of arms available to remove $i$ itself. Consider now any model $\bar{\theta}_i \in \Theta^*_i$ that must be removed from the confidence set in order to eliminate $i$. We have two cases.

\paragraph{1) $\bar{\theta}_i$ is an optimistic model w.r.t. $\theta^*$} This implies that $\mu^*(\bar{\theta}_i) = \mu_i(\bar{\theta}_i) > \mu^*(\theta^*)$ which, in turns, implies that $\Gamma_i(\bar{\theta}_i, \theta^*) > \Delta_i(\theta^*)$. Therefore, the regret for such arms can be upper bounded by
\begin{align*}
\frac{c \Delta_{i}(\theta^*) \log n}{\max_{j \in \underline{\mathcal{A}}_{\bar{h}_i} \cup \{i\}} \Gamma_j^2(\bar{\theta}_i, \theta^*)} + c' \leq \frac{c \Delta_{i}(\theta^*) \log n}{\Gamma_i^2(\bar{\theta}_i, \theta^*)} + c' \leq \frac{c \log n}{\Delta_i(\theta^*)} + c'.
\end{align*}
\paragraph{2) $\bar{\theta}_i$ is not an optimistic model w.r.t. $\theta^*$} This implies that $\mu^*(\bar{\theta}_i) = \mu_i(\bar{\theta}_i) \leq \mu^*(\theta^*)$. If $\mu_i(\bar{\theta}_i) \geq \mu^*(\theta^*) - \frac{\Delta_i}{2}$, then $\Gamma_i(\bar{\theta}_i, \theta^*) \geq \frac{\Delta_i}{2}$. If, on the other hand, $\mu_i(\bar{\theta}_i) \leq \mu^*(\theta^*) - \frac{\Delta_i}{2}$, then $\Gamma_{i^*}(\bar{\theta}_i, \theta^*) \geq \frac{\Delta_i}{2}$ since $\mu_{i^*}(\bar{\theta}_i) < \mu_i(\bar{\theta}_i)$. Furthermore, under event $E$, $i^* \in \tilde{\mathcal{A}}_h$ for all $h\geq 0$ (and thus $i^* \in \underline{\mathcal{A}}_h$). Therefore,
\begin{align*}
\frac{c \Delta_{i}(\theta^*) \log n}{\max_{j \in \underline{\mathcal{A}}_{\bar{h}_i} \cup \{i\}} \Gamma_j^2(\bar{\theta}_i, \theta^*)} + c' \leq \frac{c \Delta_{i}(\theta^*) \log n}{\max\left\{\Gamma_i^2(\bar{\theta}_i, \theta^*), \Gamma_{i^*}^2(\bar{\theta}_i, \theta^*)  \right\}} + c' \leq \frac{2c \log n}{\Delta_i(\theta^*)} + c'.
\end{align*}
This concludes the proof.

\end{proof}

\subsection{Proof of Proposition \ref{prop:comparison-SUCB}}

\propcomparisonsucb*

\begin{proof}
In the proof of Proposition \ref{prop:comparison-UCB}, we have already shown that the model gaps w.r.t. optimistic models are always larger than the action gaps. Therefore,
\begin{align*}
\inf_{\theta\in\Theta_i^* \setminus \Theta_i^+}\max_{j\in\A_i^*}\Gamma_j(\theta,\theta^*) \geq \inf_{\theta\in\Theta_i^+}\max_{j\in\A_i^*}\Gamma_j(\theta,\theta^*) \geq \Delta_i(\theta^*).
\end{align*}
The proof follows straightforwardly.
\end{proof}

\section{Proofs of Section \ref{sec:anytime}}\label{app:proofs4}

Throughout this section, we override the notation of the previous results to account for the periods introduced in Algorithm \ref{alg:isucb-anytime}. We use $T_i(k,h)$ to denote the number of pulls of arm $i$ at the end of phase $h$ in period $k$. Furthermore, we define $T_{i,k}$ as the number of pulls of $i$ at the end of period $k$. Similarly, $T_{i,k}(h)$ denotes the number of pulls of $i$ at end of phase $h$ but counting only those pulls occurred in period $k$. For all other period- and phase-dependent random variables, we shall use a superscript $k$ to denote the period and a subscript $h$ to denote the phase. For variables depending only on the period, we shall move $k$ to a subscript. We will make these dependencies explicit whenever not clear from the context. 

\subsection{Proof of Theorem \ref{th:anytime-main}}

We first extend Lemma \ref{lemma:conf} to bound the probability that the true model is not contained in the confidence set by a margin in some phase of period $k$.

\begin{lemma}\label{lemma:isucb-anytime-conf-phase}
Let $\alpha >0$, $\beta \geq 1$, $k \geq 0$, and $E_k$ denote the following event:
\begin{equation}
E_k := \left\{ \forall h=0,\dots,\lceil \log_2 \tilde{n}_k \rceil : \theta^* \in \tilde{\Theta}_{h}^k \right\}.
\end{equation}
Then, the probability that $E_k$ does not hold can be upper bounded by
\begin{equation*}
\prob{E_k^c} \leq |\mathcal{A}^*(\Theta)| (\log_2 \tilde{n}_k + 3)^2 \tilde{n}_k^{-2\frac{\alpha}{\beta^2}}\sum_{k'=0}^{k-1}\tilde{n}_{k'}.
\end{equation*}
\end{lemma}

\begin{proof}
First assume that $k > 0$. Using the union bound, we have
\begin{align*}
\prob{E_k^c} &= \prob{\exists h=0,\dots,\lceil \log_2 \tilde{n}_k  \rceil, \exists i \in \mathcal{A} : |\hat{\mu}_{i,h-1}^k - \mu_i(\theta^*)| \geq \frac{1}{\beta} \sqrt{\frac{\alpha \log \tilde{n}_k }{T_i(k,h-1)}} \wedge T_i(k,h-1) > 0}\\ &\leq \sum_{h=0}^{\lceil \log_2 \tilde{n}_k  \rceil} \sum_{i \in \mathcal{A}^*(\Theta)} \prob{|\hat{\mu}_{i,h-1}^k - \mu_i(\theta^*)| \geq \frac{1}{\beta} \sqrt{\frac{\alpha \log \tilde{n}_k }{T_i(k,h-1)}} \wedge T_i(k,h-1) > 0},
\end{align*}
where $\mathcal{A}$ can be replaced by $\mathcal{A}^*(\Theta)$ since arms that are sub-optimal for all models are never pulled and so the corresponding event above never holds. Let us now consider the inner term for a fixed phase $h$ and arm $i$. The number of pulls of $i$ can be decomposed into $T_i(k,h-1) = T_{i,k-1} + T_{i,k}(h-1)$. $T_{i,k-1}$ could be any value $s$ between $1$ and $\bar{s}_k := \sum_{k'=0}^{k-1}\tilde{n}_{k'}$. 
On the other hand, $T_{i,k}(h-1)$ can lead only to $h+1$ different number of pulls,
\begin{align*}
    p_u := \left\lceil\frac{\alpha\log \tilde{n}_k}{\tilde{\Gamma}_{u-1}^2}\left(1 + \frac{1}{\beta}\right)^2\right\rceil
\end{align*}
for $u = 1,\dots,h$ and $p_u = 0$ for $u=0$. Therefore, the number of pulls of $i$ given $s$ pulls up to period $k-1$ and $p_u$ pulls in period $k$ are $q_{s,u} = \max\{s, p_u\}$. Thus, by taking a further union bound on the possible values of $T_i(k,h-1)$ and using Chernoff-Hoeffding inequality, we obtain
\begin{align*}
\prob{|\hat{\mu}_{i,h-1}^k - \mu_i(\theta^*)| \geq \frac{1}{\beta} \sqrt{\frac{\alpha \log \tilde{n}_k}{T_i(k,h-1)}}} &= \prob{\bigcup_{s=1}^{\bar{s}_k}\bigcup_{u=0}^{h} |\hat{\mu}_{i,q_{s,u}} - \mu_i(\theta^*)| \geq \frac{1}{\beta} \sqrt{\frac{\alpha \log \tilde{n}_k}{q_{s,u}}} }\\ &\leq \sum_{s=1}^{\bar{s}_k}\sum_{u=0}^{h} \prob{|\hat{\mu}_{i,q_{s,u}} - \mu_i(\theta^*)| \geq \frac{1}{\beta} \sqrt{\frac{\alpha \log \tilde{n}_k}{q_{s,u}}}}\\ &\leq \sum_{s=1}^{\bar{s}_k}\sum_{u=0}^{h} 2e^{-2q_{s,u}\frac{\alpha\log \tilde{n}_k}{\beta^2 q_{s,u}}} = 2 (h+1) \tilde{n}_k^{-2\frac{\alpha}{\beta^2}}\bar{s}_k.
\end{align*}
Notice that, with some abuse of notation, we define $\hat{\mu}_{i,s}$ as the empirical mean of arm $i$ after $s$ pulls of such arm. Putting everything together,
\begin{align*}
\prob{E_k^c} &\leq \sum_{h=0}^{\lceil \log_2 \tilde{n}_k  \rceil} \sum_{i \in \mathcal{A}^*(\Theta)} 2 (h+1) \tilde{n}_k^{-2\frac{\alpha}{\beta^2}}\bar{s}_k = 2 |\mathcal{A}^*(\Theta)| \tilde{n}_k^{-2\frac{\alpha}{\beta^2}}\bar{s}_k \sum_{h=0}^{\lceil \log_2 \tilde{n}_k  \rceil} (h+1) \leq |\mathcal{A}^*(\Theta)| (\log_2 \tilde{n}_k + 3)^2 \tilde{n}_k^{-2\frac{\alpha}{\beta^2}}\bar{s}_k.
\end{align*}
Notice that for $k=0$ the bound is even smaller since we can avoid the union bound over the pulls in previous periods. This concludes the proof.
\end{proof}

\thanytimemain*

\begin{proof}
Let $L_k := \sum_{t=\bar{s}_k+1}^{\tilde{n}_k} \Delta_{I_t}(\theta^*)$, with $\bar{s}_k := \sum_{k'=0}^{k-1}\tilde{n}_{k'}$, be the regret incurred in period $k$. Then,
\begin{align*}
    R_n &= \expec{\sum_{t=1}^n \Delta_{I_t}(\theta^*)} \stackrel{(a)}{\leq} \expec{\sum_{k=0}^{\bar{k}} L_k} = \expec{\sum_{k=0}^{\bar{k}} L_k \indi{E_k = 1}} + \expec{\sum_{k=0}^{\bar{k}} L_k \indi{E_k = 0}}\\ &\stackrel{(b)}{\leq} \underbrace{\sum_{k=0}^{\bar{k}} \expec{L_k | E_k=1}}_{(i)} + \underbrace{\sum_{k=0}^{\bar{k}} \prob{E_k=0} \tilde{n}_k}_{(ii)},
\end{align*}
where (a) follows from the definition of the maximum period $\bar{k} = \min_{k \in \mathbb{N}^+}\{k | \tilde{n}_k \geq n\}$ and (b) by bounding the regret of each period by $\tilde{n}_k$. We now bound the two terms separately.

Let us start from (i). Fix a period $k$. We have
\begin{align*}
    \expec{L_k | E_k=1} = \sum_{i \in \mathcal{A}^*(\Theta)}\Delta_i(\theta^*)\expec{T_{i,k} - T_{i,k-1} | E_k=1},
\end{align*}
where we recall $T_{i,k}$ is the total number of pulls of $i$ at the end of period $k$ (not necessarily only in period $k$), so that $T_{i,k} - T_{i,k-1}$ is the total number of pulls occurred in period $k$. Fix a sub-optimal arm $i$. Let
\begin{align*}
    \bar{h}_{i} := \min_{h \in \mathbb{N}^+}\left\{h\ |\ \tilde{\Gamma}_h \leq \inf_{\theta \in \Theta^*_i}\max_{j \in \{i,i^*\}}\Gamma_j(\theta,\theta^*)\right\}.
\end{align*}
Lemma \ref{lemma:arm-el}, together with the fact that $i^*$ is pulled in all phases, ensures that if $i \in \tilde{\mathcal{A}}_{\bar{h}_{i}}^k$, $i$ will not be pulled again in period $k$. Therefore,
\begin{align*}
T_{i,k} - T_{i,k-1} &\stackrel{(a)}{\leq} \left\lceil\frac{\alpha\log \tilde{n}_k}{\tilde{\Gamma}_{\bar{h}_{i}}^2}\left(1 + \frac{1}{\beta}\right)^2\right\rceil \stackrel{(b)}{=} \left\lceil\frac{4\alpha\log n}{\tilde{\Gamma}_{\bar{h}_{i} - 1}^2}\left(1 + \frac{1}{\beta}\right)^2\right\rceil\\ &\stackrel{(c)}{\leq} \left\lceil\frac{4\alpha\log \tilde{n}_k}{\inf_{\theta\in\Theta^*_i} \max_{j \in \{i,i^*\}} \Gamma_j^2(\theta, \theta^*)}\left(1 + \frac{1}{\beta}\right)^2\right\rceil \stackrel{(d)}{\leq} \frac{16\alpha\log \tilde{n}_k}{\inf_{\theta\in\Theta^*_i} \max_{j \in \{i,i^*\}} \Gamma_j^2(\theta, \theta^*)} + 1\\ &\stackrel{(e)}{\leq} \frac{24\alpha\log \tilde{n}_k}{\inf_{\theta\in\Theta^*_i} \max_{j \in \{i,i^*\}} \Gamma_j^2(\theta, \theta^*)},
\end{align*}
where (a) follows from the previous comments, (b) from $\tilde{\Gamma}_{h} = \frac{\tilde{\Gamma}_{h-1}}{2}$, (c) from the definition of $\bar{h}_{i}$, (d) after setting $\beta = 1$, and (e) by noticing that $1 \leq \frac{3}{2} \log \tilde{n}_k$ for all $k \geq 0$. This allows us to bound the expected regret due to arms in $\mathcal{A}^*(\Theta)$ by
\begin{align*}
    (i) \leq \sum_{i \in \mathcal{A}^*(\Theta)} \frac{24\alpha \Delta_i(\theta^*)\sum_{k=0}^{\bar{k}} \log \tilde{n}_k}{\inf_{\theta\in\Theta^*_i} \max_{j \in \{i,i^*\}} \Gamma_j^2(\theta, \theta^*)} \leq \sum_{i \in \mathcal{A}^*(\Theta)} \frac{96\alpha \Delta_i(\theta^*) \log n}{\inf_{\theta\in\Theta^*_i} \max_{j \in \{i,i^*\}} \Gamma_j^2(\theta, \theta^*)}.
\end{align*}
To understand the second inequality, notice that $\tilde{n}_k = 2^{2^k}$ for all $k \geq 0$ since $\eta = 1$. Furthermore, since $\bar{k} < \log_{2}\log_2 n + 1$, $\sum_{k=0}^{\bar{k}}\log \tilde{n}_k = (\log 2)\sum_{k=0}^{\bar{k}} 2^k \leq 2^{\bar{k} + 1} \log 2\leq 4\log n$.

Let us now consider (ii). We have

\begin{align*}
(ii) &\stackrel{(a)}{\leq} |\mathcal{A}^*(\Theta)| \sum_{k=0}^{\bar{k}}  \tilde{n}_k ^{2 - 2\frac{\alpha}{\beta^2}} (\log_2 \tilde{n}_k  + 3)^2 \stackrel{(b)}{=}  |\mathcal{A}^*(\Theta)| \sum_{k=0}^{\bar{k}}  2^{2^k(2 - 2\frac{\alpha}{\beta^2})} (2^k  + 3)^2\\ &\stackrel{(c)}{\leq} |\mathcal{A}^*(\Theta)| \sum_{k=0}^{2}  2^{2^k(2 - 2\frac{\alpha}{\beta^2})} (2^k  + 3)^2 + |\mathcal{A}^*(\Theta)| \sum_{k=3}^{\infty} \frac{1}{2^{2^k(2\frac{\alpha}{\beta^2} - 3)}}\\ &\stackrel{(d)}{\leq} 5.76|\mathcal{A}^*(\Theta)| + 0.026|\mathcal{A}^*(\Theta)| \leq 6  |\mathcal{A}^*(\Theta)| ,
\end{align*}
where (a) follows from Lemma \ref{lemma:isucb-anytime-conf-phase} and $\sum_{k'=0}^{k-1}\tilde{n}_{k'} \leq \tilde{n}_k$, (b) from the definition of $\tilde{n}_k$, (c) from the fact that for $k \geq 3$ we have $(2^k+3)^2 \leq 2^{2^k}$, and (d) after setting $\alpha = 2$, $\beta = 1$, and some numerical calculations.

Combining (i) and (ii), we obtain the stated bound on $R_n$.

\end{proof}

\subsection{Proof of Theorem \ref{th:anytime-main-const}}

\thanytimemainconst*

\begin{proof}
As for Theorem \ref{th:anytime-main}, we define $L_k := \sum_{t=\bar{s}_k+1}^{\tilde{n}_k} \Delta_{I_t}(\theta^*)$ to be the regret incurred in period $k$. Similarly to \cite{lattimore2014bounded}, we decompose the expected regret into that incurred up to a fixed (constant in $n$) period $\underline{k}$ and that incurred in the remaining periods. Let $O_k := \{ \exists i \neq i^* : i \in \tilde{\mathcal{A}}_{0}^k \}$ be the event under which some sub-optimal arm is pulled in period $k$. Then,
\begin{align*}
    R_n &= \expec{\sum_{t=1}^n \Delta_{I_t}(\theta^*)} \stackrel{(a)}{\leq} \expec{\sum_{k=0}^{\bar{k}} L_k} \stackrel{(b)}{\leq} \sum_{k=0}^{\bar{k}} \expec{L_k | E_k=1} + \sum_{k=0}^{\bar{k}} \prob{E_k=0} \tilde{n}_k\\ &\stackrel{(c)}{=} \sum_{k=0}^{\underline{k}} \expec{L_k | E_k=1} + \sum_{k=\underline{k}+1}^{\bar{k}} \expec{L_k | E_k=1} + \sum_{k=0}^{\bar{k}} \prob{E_k=0} \tilde{n}_k\\ &\stackrel{(d)}{\leq} \underbrace{\sum_{k=0}^{\underline{k}} \expec{L_k | E_k=1}}_{(i)} + \underbrace{\sum_{k=\underline{k}+1}^{\bar{k}} \tilde{n}_k \prob{O_k = 1 | E_k = 1}}_{(ii)} + \underbrace{\sum_{k=0}^{\bar{k}} \prob{E_k=0} \tilde{n}_k}_{(iii)},
\end{align*}
where (a) and (b) are as in the proof of Theorem \ref{th:anytime-main}, (c) is trivial, and (d) follows since if $O_k = 0$ then only the optimal arm is pulled in period $k$ and thus no regret is incurred.

Using exactly the same argument as done in the proof of Theorem \ref{th:anytime-main},
\begin{align*}
    (i) \leq \sum_{i \in \mathcal{A}^*(\Theta)} \frac{24\alpha \Delta_i(\theta^*)\sum_{k=0}^{\underline{k}} \log \tilde{n}_k}{\inf_{\theta\in\Theta^*_i} \max_{j \in \{i,i^*\}} \Gamma_j^2(\theta, \theta^*)}.
\end{align*}
Similarly, we obtain $(iii) \leq 3|\mathcal{A}^*(\Theta)|$, where the smaller constant is due to the fact that we increased $\alpha$.

Let us now deal with (ii). First, we define $\underline{k}$ as
\begin{align*}
    \underline{k} := \min_{k \in \mathbb{N}^+} \left\{k\ \Big|\ \left\lfloor \frac{\tilde{n}_k}{|\mathcal{A}^*(\Theta)|}\right\rfloor \geq \frac{10\log \tilde{n}_{k+1}}{\Gamma_*^2} \right\}.
\end{align*}
By the union bound,
\begin{align*}
    (ii) &\leq \sum_{k=\underline{k}+1}^{\bar{k}} \tilde{n}_k \prob{O_k = 1 \wedge E_{k-1} = 1 | E_k = 1} + \sum_{k=\underline{k}+1}^{\bar{k}} \tilde{n}_k \prob{O_k = 1 \wedge E_{k-1} = 0 | E_k = 1}\\ &\leq \underbrace{\sum_{k=\underline{k}+1}^{\bar{k}} \tilde{n}_k \prob{O_k = 1| E_k = 1 \wedge E_{k-1} = 1}}_{(iv)} + \underbrace{\sum_{k=\underline{k}+1}^{\bar{k}} \tilde{n}_k \prob{E_{k-1} = 0 | E_k = 1}}_{(v)}.
\end{align*}
By recalling that $\tilde{n}_k = \tilde{n}_{k-1}^2$ and that $\alpha$ was increased to $\frac{5}{2}$, (v) can be bounded by $6|\mathcal{A}^*(\Theta)|$ as done for (iii) in Theorem \ref{th:anytime-main}. It only remains to bound (iv). Fix a period $k \geq \underline{k} + 1$. We have
\begin{align*}
    \prob{O_k = 1| E_k = 1 \wedge E_{k-1} = 1} &= \prob{\exists i \neq i^* : i \in \tilde{\mathcal{A}}_{0}^k | E_k = 1 \wedge E_{k-1} = 1}\\ &\stackrel{(a)}{=} \prob{\exists i \neq i^* : i \in \mathcal{A}^*(\tilde{\Theta}_{0}^k) | E_k = 1 \wedge E_{k-1} = 1}\\ &\stackrel{(b)}{\leq} \prob{T_{i^*,k-1} < \frac{\alpha\log \tilde{n}_k}{\Gamma_*^2}\left(1+\frac{1}{\beta}\right)^2 | E_k = 1 \wedge E_{k-1} = 1}\\ &\stackrel{(c)}{\leq} \prob{T_{i^*,k-1} < \left\lfloor \frac{\tilde{n}_{k-1}}{|\mathcal{A}^*(\Theta)|} \right\rfloor | E_k = 1 \wedge E_{k-1} = 1} \stackrel{(d)}{\leq} 0,
\end{align*}
where (a) follows from the definition of $\tilde{\mathcal{A}}_{0}^k$. In (b) we exploit the fact that, under event $E_k$, if $i^*$ is pulled more than that quantity at the end of period $k-1$ then no model with a different optimal arm than $i^*$ belongs to $\tilde{\Theta}_{0}^k$. (c) is from the definition of $\underline{k}$ and $k-1 \geq \underline{k}$. (d) holds since, under $E_{k-1}$, $i^*$ is pulled in all phases in period $k-1$. Therefore, even if all other arms are pulled as well, the round robin schedule of the pulls ensures $T_{i^*,k-1} \geq \left\lfloor \frac{\tilde{n}_{k-1}}{|\mathcal{A}^*(\Theta)|} \right\rfloor$.

Therefore, $(ii) \leq 6|\mathcal{A}^*(\Theta)|$. Combining (i), (ii), and (iii) we obtain
\begin{align*}
    R_n \leq \sum_{i \in \mathcal{A}^*(\Theta)} \frac{24\alpha \Delta_i(\theta^*)\sum_{k=0}^{\underline{k}} \log \tilde{n}_k}{\min_{\theta\in\Theta^*_i} \max_{j \in \{i,i^*\}} \Gamma_j^2(\theta, \theta^*)} + 9|\mathcal{A}^*(\Theta)|.
\end{align*}
Since $\sum_{k=0}^{\underline{k}} \log \tilde{n}_k = \log 2 \sum_{k=0}^{\underline{k}} 2^k \leq 2^{\underline{k} + 1}\log 2$, let us finally bound $\underline{k}$. From its definition,
\begin{align*}
    \left\lfloor \frac{2^{2^{\underline{k}-1}}}{|\mathcal{A}^*(\Theta)|}\right\rfloor < \frac{20\log 2^{2^{\underline{k} - 1}}}{\Gamma_*^2} \quad \implies \quad \frac{2^{2^{\underline{k}-1}}}{2^{\underline{k} - 1}}\leq \frac{20 |\mathcal{A}^*(\Theta)|\log 2}{\Gamma_*^2} + 2|\mathcal{A}^*(\Theta)|.
\end{align*}
Since $\underline{k} - 1 \leq 2^{\underline{k} - 2}$, we obtain
\begin{align*}
    \underline{k} \leq \log_2 \log_2 \left( \frac{20 |\mathcal{A}^*(\Theta)|\log 2}{\Gamma_*^2} + 2|\mathcal{A}^*(\Theta)| \right) + 2.
\end{align*}
Therefore,
\begin{align*}
    \sum_{k=0}^{\underline{k}} \log \tilde{n}_k \leq 2^{\underline{k} + 1}\log 2 \leq 8 \log \left( \frac{20 |\mathcal{A}^*(\Theta)|\log 2}{\Gamma_*^2} + 2|\mathcal{A}^*(\Theta)| \right),
\end{align*}
which concludes the proof.

\end{proof}

\section{Proof of the Lower Bound}\label{app:lb}

\thlb*

\begin{proof}
Throughout the proof, we consider Gaussian bandits with $\sigma^2 = \frac{1}{2}$, i.e., $\nu_i(\theta) = \mathcal{N}(\mu_i(\theta),\frac{1}{2})$ for all arms $i$ and models $\theta$. Let us fix the true model $\theta^*$ with optimal arm $i^*$ and a sub-optimal arm $i$ (such that $\Delta_i(\theta^*) > 0$). 
We build an alternative model $\theta$ as follows; for some $\epsilon$ with $0 < \epsilon < \Delta_{\text{min}}(\theta^*)$, we set the mean return of $i^*$ to either $\mu_{i^*}(\theta^*) + \epsilon$ or $\mu_{i^*}(\theta^*) - \epsilon$. This implies $\Gamma_{i^*}(\theta,\theta^*) = \epsilon$. Furthermore, we make arm $i$ become optimal, i.e., $\mu_i(\theta) > \mu_{i^*}(\theta)$. Any other arm different than $i$ and $i^*$ remains unchanged. Note that, by definition of $\epsilon$, $i^*$ is the second best arm in $\theta$.

By applying Equation 6 of \cite{garivier2018explore} together with the closed-form of the KL-divergence between Gaussians, we obtain
\begin{align}\label{eq:kl-lb}
    \notag\E_{\theta^*}[T_i(n)]\text{KL}(\nu_i(\theta^*),\nu_i(\theta))& + \E_{\theta^*}[T_{i^*}(n)]\text{KL}(\nu_{i^*}(\theta^*),\nu_{i^*}(\theta))\\ &= \E_{\theta^*}[T_i(n)]\Gamma_i^2(\theta,\theta^*) + \E_{\theta^*}[T_{i^*}(n)]\epsilon^2 \geq \text{kl}(\E_{\theta^*}[Z],\E_{\theta}[Z]),
\end{align}
where $Z$ is any random variable (measurable with respect to the $n$-step history) taking values in $[0,1]$ and $\text{kl}$ is the KL divergence between Bernoulli distributions. Choosing $Z = \frac{T_{i^*}(n)}{n}$ and using the super-fast convergence of the chosen strategy, 
\begin{equation*}
    \E_{\theta^*}[Z] = \E_{\theta^*}\left[\frac{T_{i^*}(n)}{n}\right] = 1 - \frac{1}{n} \sum_{i\neq i^*} \E_{\theta^*}[T_i(n)] \geq 1 - \frac{1}{n} \sum_{i\neq i^*} \frac{c\log n}{\Delta_i^2(\theta^*)},
\end{equation*}
and
\begin{equation*}
    \E_{\theta}[Z] = \E_{\theta}\left[\frac{T_{i^*}(n)}{n}\right] \leq \frac{c\log n}{\Delta_{i^*}^2(\theta)n} \leq \frac{c\log n}{\Delta^2 n}.
\end{equation*}
Here we defined $\Delta := \inf_{\theta' \in \Theta \setminus \Theta_{i^*}^*}\Delta_{i^*}(\theta')$. Using $\text{kl}(p,q) \geq p\log\frac{1}{q} -\log 2$,
\begin{equation*}
     \text{kl}(\E_{\theta^*}[Z],\E_{\theta}[Z]) \geq \left( 1 - \frac{1}{n} \sum_{i\neq i^*} \frac{c \log n}{\Delta_i^2(\theta^*)}\right) \log \frac{\Delta^2n}{c\log n} - \log 2.
\end{equation*}
Combining this result with \eqref{eq:kl-lb} and using the fact that the number of pulls is upper-bounded by $n$, we obtain
\begin{equation*}
        \E_{\theta^*}[T_i(n)]\Gamma_i^2(\theta,\theta^*) \geq \underbrace{\left( 1 - \frac{1}{n} \sum_{i\neq i^*} \frac{c\log n}{\Delta_i^2(\theta^*)}\right) \log \frac{\Delta^2n}{c\log n} - \log 2}_{f_i(n)} - \epsilon^2n.
\end{equation*}
Rearranging and optimizing for $\theta$,
\begin{equation*}
    \E_{\theta^*}[T_i(n)] \geq \frac{f_i(n) - \epsilon^2n}{\inf_{\theta \in \Theta_i^\epsilon}\Gamma_i^2(\theta,\theta^*)},
\end{equation*}
where $\Theta_i^\epsilon := \{ \theta\in\Theta_i^*\ |\ \Gamma_{i^*}(\theta,\theta^*) = \epsilon \}$. Following \cite{degenne2018bandits}, we use the intuition that, since the strategy is super-fast convergent, this constraint should be valid for all $t=1,\dots,n$ rather than only $n$. Therefore, since the number of pulls is monotone in $t$, 
\begin{align*}
\E_{\theta^*}[T_i(n)] \geq \sup_{1 \leq t \leq n} \frac{f_i(t) - \epsilon^2 t}{\inf_{\theta \in \Theta_i^\epsilon}\Gamma_i^2(\theta,\theta^*)}.
\end{align*}
Let us now analyze the function $f_i(t) - \epsilon^2 t$ for the particular value $t = \frac{1}{\epsilon^2}$. For $\epsilon \leq \sqrt{\frac{1}{\omega(2cd)}}$, with $d = \sum_{i\neq i^*} \frac{1}{\Delta_i^2(\theta^*)}$, we have that
\begin{equation*}
    1 - \epsilon^2 cd\log \frac{1}{\epsilon^2} \geq \frac{1}{2}.
\end{equation*}
The function $\omega$ is the one defined by \cite{lattimore2014bounded} as $\omega(x) = \min_{y \in \mathbb{N}} \{y\ |\ z \geq x\log z\ \forall z \geq y\}$. Therefore,
\begin{equation*}
    f_i\left(\frac{1}{\epsilon^2} \right) - 1 \geq \frac{1}{2} \log \frac{\Delta^2}{4 e^2 c \epsilon^2 \log \frac{1}{\epsilon^2}}.
\end{equation*}
For $n \geq \frac{1}{\epsilon^2}$ we have
\begin{equation*}
    \E_{\theta^*}[T_i(n)] \geq \frac{\log \frac{\Delta^2}{4 e^2 c \epsilon^2 \log \frac{1}{\epsilon^2}}}{2\inf_{\theta\in\Theta_i^\epsilon }\Gamma_i^2(\theta,\theta^*)}.
\end{equation*}
Applying this argument for all other sub-optimal arms, we obtain the following lower bound on the expected regret:
\begin{equation*}
    R_n(\theta^*) \geq \sum_{i \neq i^*} \frac{\Delta_i(\theta^*)}{2\inf_{\theta\in\Theta_i^\epsilon }\Gamma_i^2(\theta,\theta^*)}\log \frac{\Delta^2}{4 e^2 c \epsilon^2 \log \frac{1}{\epsilon^2}}.
\end{equation*}
Note that this hold for all $\epsilon$ such that $n \geq \frac{1}{\epsilon^2}$, $\epsilon \leq \sqrt{\frac{1}{\omega(2cd)}}$ (which also imply $\epsilon < \Delta_{\text{min}}(\theta^*)$), and for any set $\Theta$ containing $\theta^*$. It only remains to build a sufficiently-hard structure. Let $\Theta$ be such that $\theta^*\in\Theta$ and, for all models $\theta$ with optimal arm different than $i^*$, we have $\Gamma_{i^*}(\theta,\theta^*) = \Gamma_*$, with sufficiently small $\Gamma_*$ to satisfy the assumptions above. Therefore, the display above holds for $\epsilon=\Gamma_*$ and $\Theta_i^{\Gamma_*} = \Theta_i^*$. This concludes the proof.
\end{proof}

\section{Additional Details on the Experiments}

\begin{figure*}[t!]
\centering
\hspace{-0.5cm}
\begin{subfigure}[t]{0.3\linewidth}
\centering
\includegraphics[height=3.5cm]{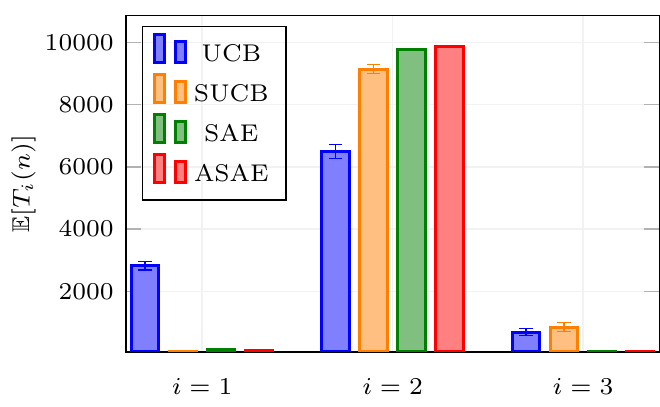}
\caption{}
\label{fig:pwlstruct_pulls}
\end{subfigure}
\hspace{0.5cm}
\begin{subfigure}[t]{0.3\linewidth}
\centering
\includegraphics[height=3.5cm]{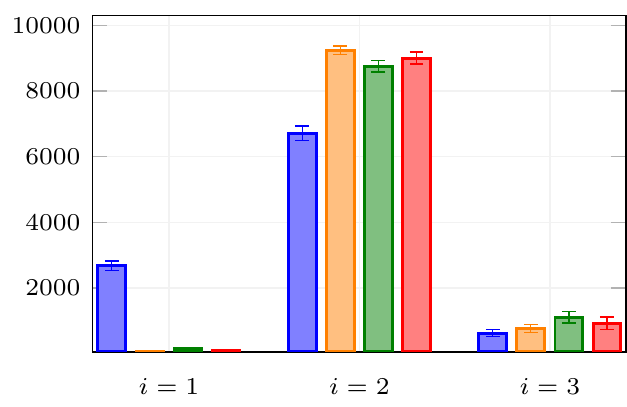}
\caption{}
\label{fig:pwlstructnoinfo_pulls}
\end{subfigure}
\hspace{0.2cm}
\begin{subfigure}[t]{0.3\linewidth}
\centering
\includegraphics[height=3.5cm]{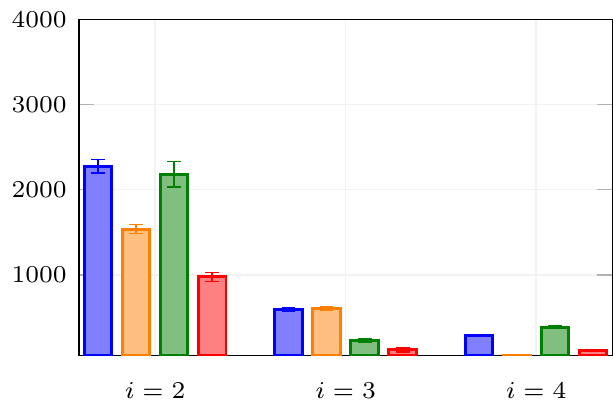}
\caption{}
\label{fig:specstruct_pulls}
\end{subfigure}
\caption{Expected number of pulls of each arm in the simulations on hand-coded structures. (a) The structure of Figure \ref{fig:example}\emph{(left)}. (b) The same structure with non-informative arm $2$. (c) The structure of Figure \ref{fig:example}\emph{(right)}. Only sub-optimal arms are shown in this last plot due the imbalanced pull counts.}
\label{fig:pulls}
\end{figure*}

We first specify the values of the means of each arm in the hand-coded structured used in the experiments. 

\paragraph*{Figure \ref{fig:example}\emph{(left)}}

\begin{itemize}
\item $\mu_1(\theta)$: from $0.85$ to $0.8$ in the first region, from $0.8$ to $0.4$ in the second, $0.4$ in the third;
\item $\mu_2(\theta)$: $0.8$ in the first region, $0.2$ in the second, $0.8$ in the third;
\item $\mu_3(\theta)$: from $0.6$ to $0.8$ in the first region, $0.86$ in the second, from $0.8$ to $0.6$ in the third;
\end{itemize}

For the simulation with non-informative arm $2$, $\mu_2(\theta) = 0.8$ for all models.

\paragraph*{Figure \ref{fig:example}\emph{(right)}}

\begin{itemize}
\item $\mu_1(\theta)$: $0.8$ in all models;
\item $\mu_2(\theta)$: $0.7$ in the first region, $0.7$ in the second, $0.4$ in the third, $0.2$ in the fourth;
\item $\mu_3(\theta)$: $0.6$ in the first region, $0.84$ in the second, $0.6$ in the third, $0.6$ in the fourth;
\item $\mu_4(\theta)$: $0.5$ in the first region, $0.1$ in the second, $0.88$ in the third, $0.5$ in the fourth;
\end{itemize}

For completeness, we report in Figure \ref{fig:pulls} the average number of pulls of each arm in the simulation of Section \ref{sec:experiments}. In Figure \ref{fig:pwlstruct_pulls}, we can notice that SAE significantly reduces the number of pulls of arm $3$ by slightly increasing those of arm $2$ (as compared to SUCB). This does not hold anymore in Figure \ref{fig:pwlstructnoinfo_pulls}, where arm $2$ became non-informative. Finally, Figure \ref{fig:specstruct_pulls} shows that, as expected, SUCB never pulls arm $4$, which however is used by SAE to significantly reduces the number of pulls to arm $2$.

\end{document}